\documentclass[11pt]{article}

\usepackage{fullpage}
\usepackage[utf8]{inputenc} 
\usepackage[T1]{fontenc}    
\usepackage{hyperref}       
\usepackage{url}            
\usepackage{booktabs}       
\usepackage{amsfonts}       
\usepackage{nicefrac}       
\usepackage{microtype}      
\usepackage{graphicx}
\usepackage{xcolor}
\usepackage{soul}
\usepackage{doi}
\usepackage{authblk}
\usepackage{amsmath}
\usepackage{amssymb}
\usepackage{enumerate}
\usepackage{epstopdf}
\usepackage{algorithmic}
\usepackage{mathtools}
\usepackage{setspace}
\usepackage{amsthm}

\theoremstyle{plain}
\newtheorem{theorem}{Theorem}
\numberwithin{theorem}{section}
\newtheorem{proposition}{Proposition}
\numberwithin{proposition}{section}
\newtheorem{lemma}{Lemma}
\numberwithin{lemma}{section}
\newtheorem{corollary}{Corollary}
\numberwithin{corollary}{section}

\numberwithin{example}{section}
\theoremstyle{definition}
\newtheorem{definition}{Definition}
\numberwithin{definition}{section}
\newtheorem{remark}{Remark}
\numberwithin{remark}{section}
\newtheorem{notation}{Notation}
\numberwithin{notation}{section}

\numberwithin{equation}{section}

\DeclareMathOperator*{\argmin}{arg\,min}
\DeclareMathOperator*{\argmax}{arg\,max}
\newcommand{\tr}[1]{\mathrm{tr}\left(#1\right)}
\newcommand{\N}{\mathbb{N}}

\usepackage{lipsum}

\providecommand{\keywords}[1]
{
  \small	
  \textbf{Keywords:} #1
}
\title{An interpretation of the Brownian bridge as a physics-informed prior for the Poisson equation}

\begin{document}

\author[1]{Alex Alberts\thanks{Corresponding author   \href{mailto:albert31@purdue.edu}{\color{blue}{\texttt{albert31@purdue.edu}}}}}
\author[1]{Ilias Bilionis}
\affil[1]{School of Mechanical Engineering, Purdue University, West Lafayette, IN}

\date{December 17, 2025}

\maketitle

\begin{abstract}
    Many inverse problems require reconstructing physical fields from limited and noisy data while incorporating known governing equations.
    A growing body of work within probabilistic numerics formalizes such tasks via Bayesian inference in function spaces by assigning a physically meaningful prior to the latent field.
    In this work, we demonstrate that Brownian bridge Gaussian processes can be viewed as a softly-enforced physics-constrained prior for the Poisson equation.
    We first show equivalence between the variational problem associated with the Poisson equation and a kernel ridge regression objective.
    Then, through the connection between Gaussian process regression and kernel methods, we identify a Gaussian process for which the posterior mean function and the minimizer to the variational problem agree, thereby placing this PDE-based regularization within a fully Bayesian framework.
    This connection allows us to probe different theoretical questions, such as convergence and behavior of inverse problems.
    We then develop a finite-dimensional representation in function space and prove convergence of the projected prior and resulting posterior in Wasserstein distance.
    Finally, we connect the method to the important problem of identifying model-form error in applications, providing a diagnostic for model misspecification.
\end{abstract}


\onehalfspacing

\section{Introduction}

A core tenant within the scientific machine learning paradigm is the development of methodologies which combine data and physics in a unified way.
In most systems of interest, along with any measurement data we also have access to some physical knowledge which the ground truth physical field is assumed to obey.
In this work, we restrict our attention to the Poisson equation with Dirchlet boundary conditions as given by
\begin{equation}
    \label{eqn:poisson}
    \left\{
    \begin{split}
        \Delta u + q &= 0 \quad \mathrm{on} \quad \Omega \subset \mathbb{R}^d \\
        u &= 0 \quad \mathrm{on} \quad \partial \Omega.
    \end{split}
    \right.
\end{equation}
We assume that the source term $q$ is sufficiently regular and $\Omega = [0,1]^d$ so that eq.~(\ref{eqn:poisson}) possesses a weak solution $u^0 \in H^1_0(\Omega)$, which denotes the subset of $H^1(\Omega)$ of functions which are also zero on the boundary.
Here by $H^1(\Omega)$, we are referring to the first-order Sobolev space.
That is, given $\tau \in \mathbb{N}$ denote the Sobolev space of square integrable functions on $\Omega$ with square-integrable weak derivatives up to order $\tau$ by $H^{\tau}(\Omega)$:
$$
H^{\tau}(\Omega) = \left\{u \in L^2(\Omega) : D^{\alpha}u \in L^2(\Omega), \:\forall\: |\alpha| \leq \tau \right\}
$$
for a multi-index $\alpha$.
The exact nature of $q$ depends on the dimension, and is discussed later.

Further, we assume we have access to some measurement data appearing in the typical way $y_i = R_iu + \gamma_i$, $i = 1, \dots, n$, where $y_i\in\mathbb{R}$ are the individual measurements, and $\gamma_i$ represents zero-mean additive noise to the measurement.
Each functional $R_i: H \to \mathbb{R}$ is called a \emph{measurement operator} and describes the process that generates the data.
At the moment, we assume that $R_i$ is continuous and linear and that the measurement noise $\gamma_i$ follows a zero-mean i.i.d. Gaussian model.
Nonlinear measurements and non-Gaussian noise are more involved in the setting of Gaussian process (GP) regression, but can be incorporated.
Here, we are interested in the derivation of a Bayesian approach for solving the Poisson equation, which treats the PDE as \emph{prior information}.
The application we have in mind is the inverse problem, where $q$ is unknown and needs to be identified.

The Poisson model underlies a wide range of applications including electrostatics and gravitation, steady-state heat conduction and diffusion, potential theory and pressure projection in incompressible flow, and imaging tasks such as Poisson image editing, shape-from-shading, and variational denoising. In many of these settings the quantity of interest (the source, a coefficient, or boundary data) is only indirectly observed through linear functionals of the state, so the measurement model above is natural. Casting the PDE itself as \emph{prior information} places our method within probabilistic numerics~\cite{hennig2015probabilistic}: the elliptic operator furnishes a Gaussian prior, via its Green’s function, that encodes regularity and boundary conditions, and Bayesian updating then yields both estimators and calibrated uncertainty.
The result is a probability distribution over the solution of the PDE.

We treat solving the PDE as an inverse problem by identifying a suitable loss function.
The first step is to cast solving eq.~(\ref{eqn:poisson}) as an optimization problem.
As we are working with the Poisson equation, we have access to a variational formulation through means of Dirichelt's principle~\cite{brezis2011functional}:
\begin{equation}
    \label{eqn:dirichlet}
    \left\{
    \begin{split}
        \min_{u \in H^1_0(\Omega)} \quad & E(u) \coloneqq \int_{\Omega} \frac{1}{2}\|\nabla u\|^2 - qu \: d\Omega \\
        \textrm{s.t.} \quad & u = 0 \quad \mathrm{on} \quad \partial\Omega.
    \end{split}
    \right.
\end{equation}
We will refer to $E(u)$ in the above as the \emph{energy functional}.
Although it is more common to use the integrated square residual of eq.~(\ref{eqn:poisson}) in much of the related literature, e.g., in physics-informed neural networks (PINNs)~\cite{raissi2019physics}, this form of the energy functional is in some sense better behaved when compared to the integrated square residual, due to the fact that it yields a convex optimization problem.
One could also view eq.~(\ref{eqn:dirichlet}) as a shifted Tikhonov regularization or equally as a smoothed total variation norm, which is common in imaging applications~\cite{rudin1992nonlinear}.
Variational forms are also sometimes used as a starting point to derive loss functions for PINNs~\cite{bai2023physics,karumuri2020simulator} and deep Ritz~\cite{yu2018deep}.

To incorporate the data, we use the usual least-squares loss:
\begin{equation}
    \label{eqn:LSE}
    \mathcal{L}_{\mathrm{data}}(u) \coloneqq \sum_{i=1}^n\left(u(x_i)-y_i\right)^2.
\end{equation}
Equations~(\ref{eqn:dirichlet}) and~(\ref{eqn:LSE}) are combined to construct the training loss function:
\begin{equation}
    \label{eqn:pinnsloss}
    \mathcal{L}(u) = \mathcal{L}_{\mathrm{data}}(u) + \eta E(u),
\end{equation}
where $\eta>0$ is a regularization parameter chosen to balance contributions from the data and physics.
Later we will see that the regularization parameter can be interpreted as a measure of model-form error.
The field reconstruction problem is then solved by minimizing this loss function.
We will refer to eq.~(\ref{eqn:pinnsloss}) as a physics-regularized inverse problem.

In the more difficult case with measurement noise, the field reconstruction problem can easily become ill-posed, so a Bayesian approach is desirable.
Methods which go about incorporating physics into the Bayesian field reconstruction problem typically do so under a GP framework.
If the PDE is linear, as in our case, it is possible to define GP priors from the physics by careful consideration of the covariance kernel.
For example, this idea can be found in~\cite{raissi2017machine,bai2024gaussian} where the covariance kernel is constructed using a numerical solver for the PDE.
There is also~\cite{harkonen2023gaussian}, which is restricted only to linear ODEs with constant coefficients.
Nonlinear PDEs can be handled by promoting the physics to the likelihood~\cite{chen2021solving} and using a standard GP prior, e.g., the square exponential kernel.
The maximum a posteriori (MAP) estimate is then taken as the solution to the PDE.
A similar work can be found in~\cite{chen2024sparse}.
This work again treats the PDE solution as a GP prior.
The solution is identified by minimizing the reproducing kernel Hilbert space (RKHS) norm of the prior covariance constrained on the PDE residuals on a predefined grid.
This of course adds additional assumptions through a regularizer which enforces smoothness and may not directly represent the underlying physics.
Also, in both methods only a deterministic answer is given.

There is also a broad class of probabilistic numerical methods which derive a prior over the PDE solution by taking the forcing term of the PDE as stochastic.
Under this stochastic PDE (SPDE) construction, we assume a random forcing $q = \xi$ (often a white noise process) and define a prior over the field $u$ through solution of the SPDE $\mathcal{A} u = \xi$, for some linear differential operator $\mathcal{A}$.
In Bayesian numerical homogenization~\cite{owhadi2015bayesian}, the prior is constructed via application of the Green's function of $\mathcal{A}$ to $\xi$ resulting in a GP whose sample paths satisfy the PDE almost surely.
Later, in~\cite{cockayne2016probabilistic}, this treatment was adapted for Bayesian inverse problems, focusing on identification of unknown coefficients but conceptually applicable to source identification.
A related approach is the statistical finite element method~\cite{girolami2021statistical}, which treats both the coefficients and the source term as GPs.
Along similar lines the widely adopted formulation in~\cite{lindgren2011explicit} uses SPDEs of the form $(\kappa^2-\Delta)^{-\alpha/2}u=\xi$ to represent a broad family of Whittle-Mat\'ern Gaussian fields with sparse Gaussian Markov random fields.
Also of note is the work in~\cite{albert2020gaussian}.
Through application of Mercer's theorem to construct a particular covariance kernel, GPs are defined whose samples are exact solutions to linear PDEs.
This idea also inspired the physics-consistent neural networks as an alternative to PINNs~\cite{ranftl2023physics}.
In our previous work~\cite{bilionis2016probabilistic}, we constructed priors over PDE solutions by directly setting the associated Green's function as the covariance kernel of the GP.
We conjectured that this process produces a physically meaningful prior and demonstrated this computationally.
This idea was adapted to the nonlinear setting following information field theory~\cite{alberts2023physics}.
Later, Poot et al. revisited this approach and made a connection to the problem of capturing model-form error induced by finite element discretization~\cite{poot2024bayesian}.
A recent review of related methods is presented in~\cite{poot2025bayesian}.

This work builds upon our previous efforts in~\cite{bilionis2016probabilistic,alberts2023physics}.
Rather than assuming a random forcing and building the prior from an SPDE, we start with the variational problem eq.~(\ref{eqn:pinnsloss}) and identify the GP prior for which this objective yields the MAP estimate.
This leads to a GP prior whose covariance is the Green's function of the Poisson equation, i.e., the Brownian bridge.
The resulting prior is therefore not a heuristic choice, as it arises naturally as the unique Gaussian prior consistent with the classical variational formulation of the Poisson equation.

This shift in view highlights a key conceptual difference between the two philosophies when selecting a prior in probabilistic numerics.
The SPDE-based approaches enforce the PDE sample-wise so that each sample from the prior produces a solution to the PDE with random forcing.
While this provides strong physical fidelity, the uncertainty is tied to the assumed forcing model and this also restricts flexibility in capturing model-form error.
In contrast, our MAP estimate derived prior imposes the physics at the level of an estimator in RKHS-norm.
The posterior mean will belong to the solution space $H^1_0(\Omega)$, while the sample paths need not satisfy the PDE.
This avoids imposing any additional regularity assumptions on the estimator.
One of the major benefits of this approach is our method's ability to capture model-form error, and in this sense our prior can be viewed as softly enforcing the physics.

\subsection{Contributions}

Our main contributions are the following:
\begin{enumerate}
    \item We show the classical Dirichlet energy functional underlying the Poisson equation arises naturally as the MAP estimator of a GP regression scheme with the Brownian bridge as the prior.
    \item We derive a finite-dimensional representation of the prior for use in applications, which places the discretization on $L^2(\Omega)$ rather than as test points in $\Omega$ as is typical in GP regression.
    This representation is provably convergent to the prior and the posterior when used in inference.
    \item We prove that this MAP estimator, and hence the function which solves eq.~(\ref{eqn:pinnsloss}), converges to the ground truth in the large-data limit.
    Convergence holds even in the presence of significant model-form error.
    \item By tuning an additional hyperparameter of the prior, we connect the method to the problem of identifying model-form error.
    We show that this hyperparameter, which controls the prior variance, is sensitive to model-form error by enforcing the physics as a soft constraint.
    The hyperparameter also causes the variance of the approximation to $q$ to adjust in the context of inverse problems.
\end{enumerate}

\subsection{Outline}

The paper is organized as follows.
In Sec.~\ref{sec:prelim}, we provide the necessary background on Gaussian process regression and kernel ridge regression.
We establish the connection between the variational problem of eq.~(\ref{eqn:pinnsloss}) and the Brownian bridge GP in Sec.~\ref{sec:physpriors}.
We do so by showing the loss function is the related kernel method objective, from which we deduce it is the MAP estimate of the corresponding GP regression.
In 1D, we also prove the result in the setting of infinite-dimensional Bayesian inverse problems.
That is, we show eq.~(\ref{eqn:pinnsloss}) is the MAP estimate of the posterior obtained when starting with the Brownian bridge as a Gaussian measure on $L^2([0,1])$.
Some analysis of the method is explored in Sec.~\ref{sec:analysis}.
Here, we state the regularity of the prior and establish convergence conditions for the MAP estimate.
We also derive a finite-dimensional approximation to the prior.
Finally, in Sec.~\ref{sec:model}, we connect the method to the problem of model-form error identification.
We demonstrate that the posterior of the inverse problem adjusts according to error in the specified physical model.

\section{Preliminaries}
\label{sec:prelim}

We provide some necessary background on GPs and RKHSs.
In Appendix~\ref{apdx:GM}, we also provide a background on the theory of Gaussian measures, which, while used somewhat, is not the main focus in this work.
\begin{definition}[Reproducing kernel Hilbert space~\cite{kanagawa2018gaussian}]
    \label{def:rkhs}
    Let $k$ be a symmetric, positive-definite function on $\Omega\times\Omega$.
    A Hilbert space $H_k$ on $\Omega$ equipped with inner product $\langle\cdot,\cdot\rangle_{H_k}$ is said to be a reproducing kernel Hilbert space if the following two properties hold:
    \begin{enumerate}
        \item For all fixed $x'\in \Omega$, $k(\cdot,x')\in H_k$.
        \item For all fixed $x'\in \Omega$ and for all $u\in H_k$, $u(x')=\langle u, k(\cdot,x')\rangle_{H_k}$.
    \end{enumerate}
\end{definition}
Property (ii) of Def.~\ref{def:rkhs} is called the \emph{reproducing property}, and the kernel defining the RKHS is called the \emph{reproducing kernel}.
The RKHS is uniquely determined by the positive-definite kernel that defines it, and the reverse is also true.
This results from the Moore-Aronszajn theorem~\cite{aronszajn1950theory}, which states that every positive definite symmetric map $k$ is associated with a unique RKHS $H_k$ for which $k$ is the reproducing kernel.
One can show that given a positive definite kernel $k$ and its RKHS, each $f\in H_k$ can be written as $f = \sum_{i=1}^\infty \alpha_i k(\cdot,x_i)$ for some $(\alpha_i)_{i=1}^{\infty}\subset \mathbb{R}$, $(x_i)_{i=1}^{\infty}\subset \Omega$ and $\|f\|_{H_k}<\infty$, where $\|f\|_{H_k}^2 \coloneqq \sum_{i,j=1}^{\infty}c_ic_jk(x_i,x_j)$~\cite[Equation (6)]{kanagawa2018gaussian}.
It is therefore easy to verify that the functions in the RKHS have the same behavior of $k$, e.g., smoothness.

In the most general case, it is quite difficult to identify the RKHS and its inner product.
However, Mercer's theorem provides an easily accessible way to characterize $H_k$.
Begin by defining the integral operator on $L^2(\Omega)$ by
\begin{equation}
    \label{eqn:intop}
    (C_ku)(x) \coloneqq \int k(x,x')u(x')dx', \quad \: u\in L^2(\Omega).
\end{equation}
The assumptions on $k$ imply that $C_k$ is a self-adjoint, positive operator, and thus has spectral decomposition
\begin{equation*}
    C_ku = \sum_{n\in\N}\lambda_n\langle u, \psi_n\rangle\psi_n,
\end{equation*}
where $(\lambda_n,\psi_n)_{n=1}^{\infty}$ is the eigensystem of $C_k$, i.e.
\begin{equation}
    \label{eqn:esystem}
    C_k\psi_n = \lambda_n\psi_n,
\end{equation}
for $n\in\mathbb{N}$, where each $\lambda_n\geq 0$ and $\lambda_n\to 0 $.
Then, Mercer's theorem provides an alternative expression for the kernel:
\begin{theorem}[\textbf{Mercer's Theorem}~\cite{steinwart2008support}]
    Let $k:\Omega \times \Omega \to \mathbb{R}$ be a continuous, positive-definite kernel, and $C_k$ and $(\lambda_n,\psi_n)_{n=1}^{\infty}$ be as given in eq.~(\ref{eqn:intop}) and eq.~(\ref{eqn:esystem}), respectively.
    Then,
    $$
    k(x,x') = \sum_{n=1}^{\infty} \lambda_n\psi_n(x)\psi_n(x'),
    $$
    for $x,x'\in\Omega$, where the convergence is absolute and uniform.
\end{theorem}
Mercer's theorem also allows an equivalent representation of the RKHS in terms of $L^2$ inner products.
That is, the RKHS is given by
$$
    H_k = \left\{u \in L^2(\Omega) : \sum_{n\in\N} \frac{1}{\lambda_n} \langle u, \psi_n\rangle^2 <\infty \right\},
$$
and the inner product on $H_k$ is
$$
    \langle u, v \rangle_{H_k} =  \sum_{n\in\N} \frac{1}{\lambda_n} \langle u, \psi_n\rangle \langle v, \psi_n\rangle,
$$
for $u,v \in H_k$.
Hence the RKHS-norm can be expressed as $\|u\|_{H_k}^2 = \sum_{n=1}^{\infty}\lambda_n^{-1}\langle u,\psi_n\rangle^2$.
This representation is useful to us later when constructing the RKHS associated with the physics-regularized inverse problem.

Next, we summarize the relationship between GP regression and kernel ridge regression (KRR) and the importance of the prior covariance RKHS in GP regression.
We start with GP regression.
Recall the definition of a GP:
\begin{definition}[Gaussian process~\cite{kanagawa2018gaussian}]
    Let $m:\Omega \to \mathbb{R}$ be a function and $k:\Omega\times \Omega\to \mathbb{R}$ be a positive definite kernel.
    The random function $u:\Omega \to \mathbb{R}$ is a Gaussian process with mean function $m$ and covariance function $k$, if for any set $X = (x_1,\dots,x_n)\subset \Omega$ for $n\in\mathbb{N}$, the random vector
    $$
    u_X\coloneqq (f(x_1),\dots,f(x_n))^T\in \mathbb{R}^n
    $$
    follows a multivariate Gaussian distribution with mean vector $m_X \coloneqq(m(x_1),\dots,m(x_n))^T$ and covariance matrix $K_{XX}$ with elements $(K_{XX})_{ij} = k(x_i,x_j)$.
    That is, $u_X\sim\mathcal{N}(m_X,K_{XX})$.
    In this case, we denote the GP by $u\sim\mathcal{GP}(m,k)$.
\end{definition}

GPs are often used in regression tasks, where in the simplest case we have point observations with zero-mean Gaussian noise.
Let $u:\Omega\to\mathbb{R}$ denote the target function and assume that we have training data in the form of
\begin{equation}
    \label{eqn:data}
    y_i = u(x_i) + \gamma_i, \quad i = 1, \dots, n,
\end{equation}
where $\gamma_i\overset{i.d.d.}{\sim}\mathcal{N}(0,\sigma^2)$, and we consolidate the observations into the data tuples $X = (x_1,\dots,x_n)$ and $y = (y_1,\dots,y_n)$.
In the GP regression approach, we start by specifying a prior GP, $u\sim\mathcal{GP}(m,k)$, where the mean and covariance function are chosen to reflect our prior knowledge about $u$.
We then define a likelihood $p(X,y|u) = \prod_{i=1}^n\mathcal{N}(y_i|u(x_i),\sigma^2)$.
The GP regression posterior is derived by conditioning the prior on the data, which also results in a GP:
\begin{theorem}[Theorem 3.1~\cite{kanagawa2018gaussian}]
    \label{thm:GPpost}
    Assume we have data given by eq.~(\ref{eqn:data}) and a GP prior $u\sim\mathcal{GP}(m,k)$.
    Then the posterior follows $u|y\sim \mathcal{GP}(\Tilde{m}, \Tilde{k})$, where
    \begin{align}
        \Tilde{m}(x) &\coloneqq m(x) + k_{xX}(K_{XX}+\sigma^2I_n)^{-1}(y-m_X), \quad x\in\Omega \label{eqn:GPmean} \\
        \Tilde{k}(x,x') &\coloneqq k(x,x') - k_{xX}(K_{XX}+\sigma^2I_n)^{-1}k_{Xx'}, \quad x,x'\in\Omega,
        \label{eqn:GPcov}
    \end{align}
    with $k_{xX} = k_{Xx}^{T}\coloneqq (k(x,x_1),\dots,k(x,x_n))^T$.
\end{theorem}
We refer to $\Tilde{m}$ as the \emph{posterior mean function} and $\Tilde{k}$ as the \emph{posterior covariance function}.

Kernel ridge regression (KRR), or regularized least squares~\cite{caponnetto2007optimal}, is closely related to GP regression.
Given data in eq.~(\ref{eqn:data}), the objective of KRR is to solve the following interpolation problem
\begin{equation}
    \label{eqn:KRRobj}
    u^* = \argmin_{u\in H_k} \frac{1}{n}\sum_{i=1}^n\left(u(x_i)-y_i\right)^2 + r\|u\|_{H_k}^2,
\end{equation}
where $r\geq 0$ is the regularization parameter.
The inclusion of the RKHS norm in the objective function serves as a regularizer which enforces the class of functions which fit the data, while simultaneously smoothing the fit.
It is known that $u$ becomes smoother as $\|u\|_{H_k}$ gets smaller, see~\cite[Section 6.2]{williams2006gaussian}.
Specifying the kernel which defines the KRR objective eq.~(\ref{eqn:KRRobj}) effectively enforces a prior on the fit.
As with the GP regression posterior mean function, the solution to eq.~(\ref{eqn:KRRobj}) is also unique:
\begin{theorem}[Theorem 3.4~\cite{kanagawa2018gaussian}]
    Let $r>0$.
    Then the unique solution to eq.~(\ref{eqn:KRRobj}) is 
    $$
    u^*(x) = k_{xX}(K_{XX} + nr I_n)^{-1}y = \sum_{i=1}^na_ik(x,x_i),\quad x\in\Omega,
    $$
    where $k_{xX} = k_{Xx}^{T}\coloneqq (k(x,x_1),\dots,k(x,x_n))^T$ and $(a_1,\dots,a_n)^T = (K_{XX} + nr I_n)^{-1}d$.
    Further, if the matrix $K_{XX}$ is invertible, then the coefficients $a_i$ are unique.
\end{theorem}

In~\cite{kanagawa2018gaussian}, the relationship between GP regression and KRR is discussed in great detail.
In a certain sense, GP regression can be viewed as the Bayesian interpretation of KRR.
Notably, under mild conditions, the KRR solution and GP posterior mean function are equivalent.
\begin{proposition}
    \label{prop:gpkrrequiv}
    Let $k:\Omega \times \Omega \to \mathbb{R}$ be a positive definite kernel, and eq.~(\ref{eqn:data}) be training data.
    If $\sigma^2 = n\lambda$, then $\Tilde{m} = u^*$, where $\Tilde{m}$ is the GP posterior mean function and $u^*$ is the unique KRR solution, given by eq.~(\ref{eqn:GPmean}) with $m=0$ and eq.~(\ref{eqn:KRRobj}), respectively.
\end{proposition}

The equivalence between the GP posterior mean and KRR solution helps to establish much of the behavior involved with GP regression in terms of the RKHS of the prior covariance kernel.
For example, it is immediate from Proposition~\ref{prop:gpkrrequiv} that the GP posterior mean function lives in the RKHS of the prior, meaning that the behavior of the posterior mean is inherited from the specified prior covariance.
The last important property we need is the fact that GP sample paths a.s. do not belong to the prior RKHS, which is a consequence of Driscol's zero-one law~\cite{driscoll1973reproducing}.
\begin{proposition}[Corollary 4.10~\cite{kanagawa2018gaussian}]
    \label{thm:GPsamples}
    Let $k: \Omega \times \Omega \to \mathbb{R}$ be a positive definite kernel and $H_k$ be the corresponding RKHS.
    Let $u \sim \mathcal{GP}(m,k)$ where $m\in H_k$.
    If $H_k$ is infinite-dimensional, then $u\in H_k$ with probability $0$.
\end{proposition}

\section{The Brownian bridge as a physics-informed prior}
\label{sec:physpriors}

We establish an explicit connection between the Brownian bridge GP estimator and the physics-regularized inverse problem.
In particular, we show the posterior mean function when starting with a shifted Brownian bridge GP is exactly the function which minimizes the variational problem eq.~(\ref{eqn:pinnsloss}), under certain criteria.
We demonstrate this for the case where the posterior remains Gaussian, and also when the posterior is non-Gaussian, as a MAP estimator in 1D.
We begin with the simpler case where we have point measurements with additive Gaussian noise according to eq.~(\ref{eqn:data}).

Before proceeding, we discuss the regularity conditions so that the definitions and operations that follow are well-defined.
Because we are modeling the solution to eq.~(\ref{eqn:poisson}) as a GP, we would ideally like to be able to evaluate it pointwise.
So, at the very least, we ask that the solution also belongs to the space of bounded, continuous functions over the domain, which we will denote by $\mathcal{C}^0_B(\Omega).$
The required conditions are given by the next result.

\begin{proposition}
    \label{prop:regularity}
    Assume that eq.~(\ref{eqn:poisson}) possesses a weak solution $u^0 \in H_{0}^1(\Omega)$.
    Further assume that if $d \leq 3$, then $q \in L^2(\Omega)$, and for $d > 3$, $q \in H^{\tau}(\Omega)$ with $\tau > d/2 - 2$.
    Then $u^0\in H^2(\Omega)$ for $d \leq 3$ and $u^0 \in H^{\tau+2}(\Omega)$ for $d > 3$.
    In both cases, there exists a version of $u^0$ that is equivalent to $u^0$ a.e., denoted by $u^{\dagger}$, with $u^{\dagger} \in C^0_{B}(\Omega)$.
\end{proposition}
\begin{proof}
    We show that $u^0$ lives in the stated Sobolev spaces, which are then embedded in $C^0_B(\Omega)$.
    The case $d\leq 3$ follows from~\cite[Sec. 6.3, Theorem 4]{evans2022partial}, and $d>3$ results from~\cite[Sec. 6.3, Theorem 5]{evans2022partial}.
    The continuity follows from a case of the Sobolev embedding theorem~\cite[Theorem 4.12]{adams2003sobolev}, which states that for integers $s\geq 0$, $\tau \geq 1$, $H^{s+\tau}(\Omega)$ is embedded in $C^s_B(\Omega)$ when $2\tau > d$.
    The result follows in each case by setting $s = 0$.
\end{proof}
Moving forward, we will assume that Proposition~\ref{prop:regularity} holds, and by the equivalence relation given by the embedding, we will take the weak solution to be $u^{\dagger}$ so that we may evaluate it pointwise.
So that there is no confusion, we will still refer to this function as $u^0$.

\subsection{Physics-informed prior as a Gaussian process}
The first step is to identify the covariance kernel hidden in the energy functional
\begin{equation}
    \label{eqn:energy}
    E(u) = \int_{\Omega}\frac{1}{2}\|\nabla u\|^2 - qu\:d\Omega.
\end{equation}
Let $L$ denote the minus Laplacian operator on $H^1_0(\Omega)\cap H^2(\Omega)$, i.e., $(Lu)(x) = - \nabla^2u|_x$.
We will see later that $L$ is the precision operator associated with the GP we are after.
Denote the inverse of $L$ by $C$.
This is the operator with kernel given by the Green's function of the Laplacian, i.e., $C$ is the operator defined by
$$
(Cu)(x) \coloneqq (L^{-1}u)(x) = \int_{\Omega}k(x,x')u(x')dx', \quad u \in L^2(\Omega).
$$
In our example where $\Omega = [0,1]^d$, the covariance kernel $k$ is best expressed with the Mercer representation.

\begin{notation}
    Before writing the covariance kernel we introduce a multi-index notation.
    Let $\alpha = (\alpha_1,\dots,\alpha_d)\in \N^d$ and $|\alpha| = \sum_{i=1}^d \alpha_i$.
    We also write for an $x = (x_1,\dots,x_d) \in \Omega$, $\sin(\alpha \pi x) = \prod_{i=1}^d\sin(\alpha_i \pi x_i)$.
\end{notation}

One can check the orthonormal eigenfunctions associated with $C$ are~\cite[Sec. 8.2.2-16]{polyanin2001handbook}
\begin{equation}
    \label{eqn:efuncs}
    \psi_{\alpha}(x) = 2^{d/2}\sin(\alpha\pi x),
\end{equation}
with corresponding eigenvalues
\begin{equation}
    \label{eqn:evals}
    \lambda_{\alpha} = \frac{1}{\pi^2|\alpha|^2}.
\end{equation}
Hence the covariance kernel has a nice tensor product structure
\begin{equation}
    \label{eqn:browniankernel}
    k(x,x') = 2^d\sum_{|\alpha|\in\N}\frac{\sin(\alpha \pi x)\sin(\alpha\pi x')}{\pi^2|\alpha|^2}.
\end{equation}
This kernel is associated with the Brownian bridge in $d$-dimensions.
For example, on $[0,1]$, the kernel is simply the Green's function of eq.~(\ref{eqn:poisson}) $k(x,x') = \min\{x,x'\} - xx'$, which is exactly the covariance kernel of the Brownian bridge process.

We now show that the physics-regularized inverse problem emits a KRR objective with this kernel.
\begin{proposition}
    \label{prop:pinnsloss}
    The physics-regularized inverse problem given by eq.~(\ref{eqn:pinnsloss}) is equivalent to the shifted kernel ridge regression objective 
    \begin{equation}
        \label{eqn:propploss}
        \mathcal{L}(u) = \frac{1}{n}\sum_{i=1}^n (u(x_i) - y_i)^2 + \frac{r}{2}\|u-Cq\|_{H_k}^2,
    \end{equation}
    with covariance kernel given by eq.~(\ref{eqn:browniankernel}).
\end{proposition}
\begin{proof}
    In both eq.~(\ref{eqn:pinnsloss}) and eq.~(\ref{eqn:propploss}), the data contribution term is exactly the same, so we only need to verify the energy functional is the correct RKHS-norm.
    An equivalent expression for the energy as given by eq.~(\ref{eqn:pinnsloss}) is the quadratic form $E(u) = \frac{1}{2}\langle u - Cq, L(u-Cq)\rangle$, which can be seen by completing the square.
    We have by integration by parts
    \begin{align*}
        \int_{\Omega}\frac{1}{2}\|\nabla u\|^2 - qu\: d\Omega &= \int_{\partial \Omega}\frac{1}{2}u\nabla u\cdot \mathbf{n}\: d\Omega - \int_{\Omega}\frac{1}{2}u\nabla^2u\: d\Omega \color{red}-\color{black} \int_{\Omega} qu\: d\Omega \\ 
        &= \int_{\Omega}\frac{1}{2}uLu - qu\: d\Omega \\
        &= \int_{\Omega}\frac{1}{2}(u-L^{-1}q)L(u-L^{-1}q)\: d\Omega + \mathrm{const.}\\
        &= \frac{1}{2}\langle u - Cq, L(u-Cq)\rangle + \mathrm{const.},
    \end{align*}
    where the surface integral vanishes due to the imposed boundary conditions.
    To connect the quadratic form to the RKHS norm, note by Mercer representation
    \begin{align*}
        (Lu)(x) &= \int_{\Omega}\sum_{|\alpha|\in\N}\lambda_{\alpha}^{-1}\psi_{\alpha}(x)\psi_{\alpha}(x')u(x')dx' \\
        &= \sum_{|\alpha|\in\N} \lambda_{\alpha}^{-1}\psi_{\alpha}(x)\langle u, \psi_{\alpha}\rangle,
    \end{align*}
    where interchanges between summation and integration are permitted by the absolute continuity of Mercer's theorem.
    Plugging this into the quadratic form, we get
    \begin{align*}
        E(u) &= \frac{1}{2}\left\langle u - Cq, \sum_{|\alpha|\in\N} \lambda_{\alpha}^{-1}\psi_{\alpha}\langle u - Cq, \psi_{\alpha}\rangle\right\rangle + \mathrm{const.} \\ 
        &= \frac{1}{2}\sum_{|\alpha|\in\N}\lambda_{\alpha}^{-1} \left\langle u - Cq, \psi_{\alpha}\langle u - Cq, \psi_{\alpha}\rangle\right\rangle + \mathrm{const.} \\
        &= \frac{1}{2}\sum_{|\alpha|\in\N}\lambda_{\alpha}^{-1} \langle u-Cq,\psi_{\alpha}\rangle^2 + \mathrm{const.} \\
        &= \frac{1}{2}\|u-Cq\|_{H_k}^2 + \mathrm{const.}
    \end{align*}
    where the last line holds by Mercer representation of the RKHS-norm.
    The additive constant may be dropped from the optimization.
\end{proof}

So, we have established a connection between KRR with the Brownian bridge kernel and the loss function coming from the variational form of the Poisson equation.
Observe that the objective function is shifted to minimize the distance in RKHS between the estimator and $Cq$.
As $C$ is defined by the Green's function, $Cq$ is the unique solution to eq.~(\ref{eqn:poisson}), and we verify that the model is indeed trying to find the closest possible match to the solution of eq.~(\ref{eqn:poisson}), while also fitting the data.
In fact, unlike the integrated square residual which seeks to minimize the $L^2(\Omega)$-norm, this objective function is also trying to match the smoothness.
This is evidenced later when we identify $H_k = H^1_0(\Omega)$.

From Proposition~\ref{prop:pinnsloss}, it is now fairly trivial to connect the physics-regularized inverse problem to a GP regression scheme.
Recall now Proposition~\ref{prop:gpkrrequiv}, which shows an equivalence between the GP posterior mean function and the KRR estimator.
That is, $\Tilde{m}$ is the function which solves the problem
\begin{equation*}
    \Tilde{m} = \argmin_{u\in H_k} \frac{1}{n}\sum_{i=1}^n\left(u(x_i)-y_i\right)^2 + \frac{\sigma^2}{n}\|u - u^0\|_{H_k}^2.
\end{equation*}
Inspired by this, we define the GP prior, which we term as the physics-informed prior for the Poisson equation, $u \sim \mathcal{GP}(u^0, \beta^{-1}k)$, $u^0 = Cq$, and $k$ is the Brownian bridge kernel as given by eq.~(\ref{eqn:browniankernel}).
In doing so we obtain a natural Bayesian interpretation of the familiar variational formulation of the Poisson equation.
We have included a hyperparameter $\beta \in (0,\infty)$ in order to control the variance of this prior.
Later we will prove that $\beta$ plays a key role in detecting model-form error.
Notice that the prior is centered at the unique solution to eq.~(\ref{eqn:poisson}).

\begin{figure}[htbp]
    \centering
    \includegraphics[width=.35\textwidth]{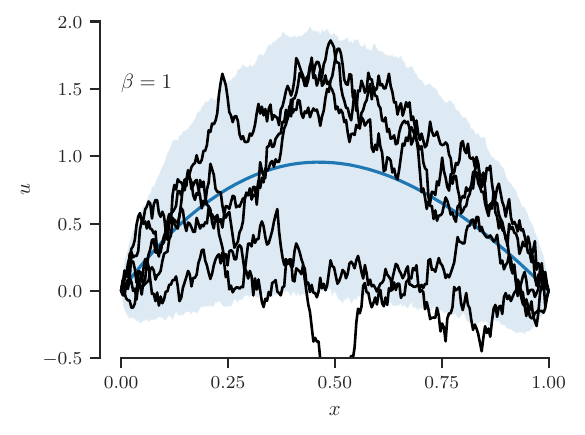}
    \includegraphics[width=.35\textwidth]{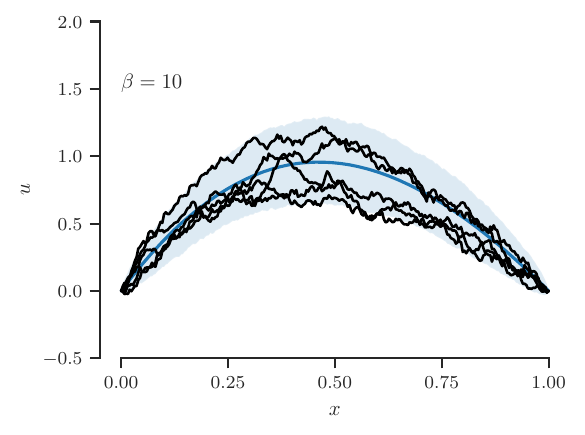}
    \hfill
    \includegraphics[width=.35\textwidth]{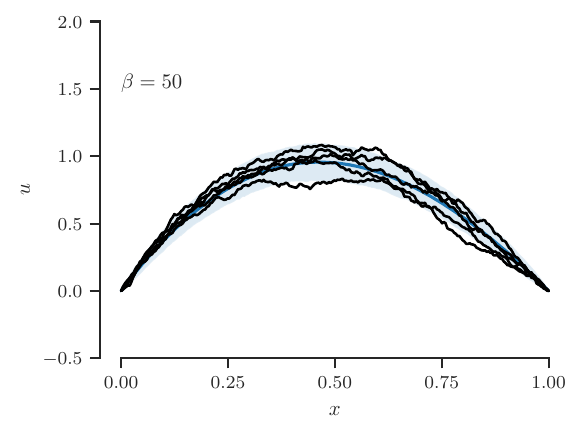}
    \includegraphics[width=.35\textwidth]{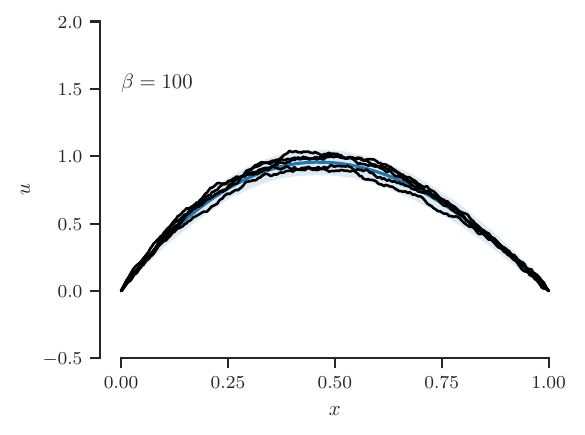}
    \caption{Physics-informed prior for the Poisson equation with source term $q(x) = 10\exp\{-|x-1/4|^2\}$ for varying values of $\beta$.
    The solution to this equation is in blue, and sample paths are shown in black.}
    \label{fig:exampleprior}
\end{figure}

The prior allows the sample paths to vary around $u^0$, which is desirable in the case of an imperfect model.
For the Brownian bridge, the variance reaches its maximum in the center of the domain, with no variance on $\partial\Omega$.
The additional hyperparameter $\beta$ controls the magnitude of the variance, which can be seen in the limiting cases.
As $\beta \to 0$, $\mathbb{V}[u]\to \infty$.
This essentially corresponds to placing a flat, uninformative prior on $L^2(\Omega)$, and the physics plays no role.
If $\beta \to \infty$, the prior collapses to a Dirac centered at $u^0$.
This corresponds to the ultimate belief that the underlying field truly is governed by the Poisson equation.
The only field we consider is the one a priori assumed to be correct.
It is for this reason that we view the prior as a \emph{soft-constraint} for the physics, with $\beta$ encoding the degree of model-trust.
An example of this behavior is shown in Fig.~\ref{fig:exampleprior}.

Setting $r = \sigma^2\beta/n$ gives an equivalence between the posterior mean function, when starting with the physics-informed prior, and the minimizer of the physics-regularized inverse problem.
We summarize this result in the following theorem.
\begin{theorem}
    \label{thm:iftpinnequiv}
    Consider training data of the form $y_i = u(x_i) + \gamma_i$, $i=1,\dots,n$, where $\gamma_i \overset{\mathrm{i.i.d.}}{\sim} \mathcal{N}(0,\sigma^2)$ and let $u:\Omega \to \mathbb{R}$ be the target function.
    Let $E$ be the energy functional for the Poisson equation, i.e. $E(u) = \int \frac{1}{2}\|\nabla u\|^2 + qu\:d\Omega$.
    Letting $r = \sigma^2\beta/n$, we have $\Tilde{m} = \hat{u}$, where
    \begin{enumerate}
        \item $\Tilde{m}$ is the GP regression posterior mean function with prior $u \sim \mathcal{GP}(u^0, \beta^{-1}k)$, where $u^0$ is the unique solution to eq.~(\ref{eqn:poisson}) and $k$ given by eq.~(\ref{eqn:browniankernel}).
        \item $\hat{u}$ is the solution to the physics-regularized inverse problem
        $$
        \hat{u} = \argmin_{u \in H^1_0(\Omega)}\: \mathcal{L}_{\mathrm{data}}(u) + r E(u).
        $$
    \end{enumerate}
\end{theorem}

The above result allows us to analyze the behavior of the physics-regularized inverse problem through the established theory of GP regression.
For example, we can study convergence conditions for the field reconstruction problem and the effect of the hyperparameters.
This is reserved for later sections.

\subsection{Physics-informed prior as a Gaussian measure}
We derive a similar result to Theorem~\ref{thm:iftpinnequiv} in the setting of infinite-dimensional Bayesian inverse problems~\cite{stuart2010inverse} in the 1D case.
This is useful, for instance, in the case where the measurement operator is nonlinear, and we cannot easily rely on the GP formulae.
We use this relationship to derive a parallel variational problem, modified to account for a general measurement operator, by identifying the functional with produces the MAP estimate of the inverse problem starting with the Brownian bridge as the prior measure.

The reason we must restrict ourselves to 1D is the fact that the covariance operator associated with the Brownian bridge kernel is trace class in $L^2(\Omega)$ only when $d = 1$.
That is, $\tr{C} = \sum_{n_1,\dots,n_d\in\N} \lambda_{n_1,\dots,n_d}<\infty$ for $d = 1$ and diverges for $d > 1$.
The GP prior is related to a Gaussian measure on $L^2(\Omega)$ in the following manner
\begin{theorem}[Theorem 2~\cite{rajput1972gaussian}]
    \label{thm:GMGP}
    Let $u \sim \mathcal{GP}(m,k)$ be measurable.
    Then, the sample paths $u \in L^2(\Omega)$ a.s. if and only if
    $$
    \int_{\Omega}m^2(x)dx < \infty, \quad \int_{\Omega}k(x,x)dx <\infty.
    $$
    If the above holds then $u$ induces the Gaussian measure $\mathcal{N}(m,C)$ on $L^2(\Omega)$, where the covariance operator is given by $\left(Cv\right)(x) \coloneqq \int_{\Omega} k(x,x')v(x')dx'$, for $v \in L^2(\Omega)$.
\end{theorem}
In the above theorem the condition $\int_{\Omega}k(x,x)dx < \infty$ is exactly the condition that $\tr{C} < \infty$.
So, in one-dimension we may interpret the physics-informed prior as the Gaussian measure $\mu_0 \sim \mathcal{N}(u^0,\beta^{-1}C)$ on $L^2(\Omega)$, and follow the infinite-dimensional Bayesian framework.

For derivations, it is often convenient to shift the space so that the prior is centered. 
According to Theorem~\ref{thm:equiv}, this is permitted so long as $u^0 \in H_k$.
As $u^0$ is exactly the solution of eq.~(\ref{eqn:dirichlet}), we have $u^0 \in H^1_0(\Omega)$.
Later in Lemma~\ref{lem:bridgerkhs}, we show that $H_k = H^1_0(\Omega)$ so that the shift is justified.

Let $\mathcal{X}$ denote the function space for which the target function lives.
Suppose now we have data $y\in \mathbb{R}^n$ generated according to $y = R(u) + \gamma$, where $R:\mathcal{X} \to \mathbb{R}^n$ is the observation map, which is in general nonlinear, and $\gamma \sim \mathcal{N}(0,\Gamma)$ is an additive noise process.
Following the Bayesian approach~\cite{stuart2010inverse}, we look to derive the posterior in function space.
To identify the posterior measure $\mu^y$, we apply Bayes's rule, 
which takes the following form in infinite-dimensions.
\begin{theorem}[Bayes's theorem~\cite{kaipio2006statistical,stuart2010inverse}]
    \label{thm:bayes}
    Let $\mu_0\sim \mathcal{N}(u^0,C)$ be the prior, and suppose that $R: \mathcal{X} \to \mathbb{R}^n$ is continuous with $\mu_0(\mathcal{X})=1$.
    Then the posterior distribution over the conditional random variable $u|y$ obeys $\mu^y \ll \mu_0$.
    It is given by the Radon-Nikodym derivative
    $$
    \frac{d\mu^y}{d\mu_0}(u) \propto \exp\left\{-\Phi(u)\right\},
    $$
    where $\Phi(u) \coloneqq \frac{1}{2}\|\Gamma^{-1}(y-R(u))\|^2$ is called the potential.
\end{theorem}
Theorem~\ref{thm:bayes} admits a closed form expression in a special case.
Assuming that $R$ is linear, the posterior $\mu^d$ is also Gaussian $\mathcal{N}(\Tilde{m},\Tilde{C})$, with
\begin{align*}
    \Tilde{m} &= u^0 + CR^{\dagger}(\Gamma + RCR^{\dagger})^{-1}(y-Ru^0) \\
    \Tilde{C} &= C - CR^{\dagger}(\Gamma + RCR^{\dagger})^{-1}RC,
\end{align*}
where $R^{\dagger}$ denotes the adjoint of $R$.

\begin{remark}
    There is a minor technicality to discuss here about the existence and interpretation of $\mu^{y}$ as a posterior measure.
    The prior measure must be chosen such that $\mu_0(\mathcal{X}) = 1$.
    In much of the literature, the measurement operator involves solving a PDE, in which case care must be taken when choosing the prior.
    The advantage of our approach is that the physics is encoded into the prior, rather than the likelihood.
    For the Brownian bridge, $\mu_0(L^2(\Omega)) = 1$, so the only requirement is that $R$ acts on $L^2(\Omega)$, a fairly trivial assumption.
\end{remark}

To identify the MAP estimate of the posterior, we follow the work laid out in~\cite{dashti2013map}.
The MAP estimate is identified through the Onsager-Machlup functional~\cite{durr1978onsager,fujita1982onsager}.
This is the functional $I: H_k \to \mathbb{R}$ such that
$$
\lim_{\varepsilon\to0} \frac{\mu^y(B(u_2;\varepsilon))}{\mu^y(B(u_1;\varepsilon))} = \exp\{I(u_1) - I(u_2)\},
$$
where $B(u_i;\varepsilon)$ is the open ball in $L^2(\Omega)$ centered at $u_i$ with radius $\varepsilon$.
For fixed $u_1$, any function $u_2$ which minimizes the Onsager-Machlup functional can be taken as the MAP estimate.
For our specific problem, the Onsager-Machlup functional is
\begin{equation}
    \label{eqn:OMfunc}
    I(u) = \left\{
        \begin{split}
            &\Phi(u) + \frac{1}{2}\|u-u^0\|^2_{H_k}, \quad \mathrm{if}\:u-u^0\in H_k \\
            &+\infty,\quad \mathrm{otherwise},
        \end{split}
        \right.
\end{equation}
as shown in~\cite[Theorem 3.2]{dashti2013map}.
So, any MAP estimate of $\mu^y$ will live in the Cameron-Martin space (which is also $H_k$) of $\mu_0$.
Further, if $\Phi$ is linear in $u$ this MAP estimate is unique.

Equation~\ref{eqn:OMfunc} is a natural candidate to build the variational problem for a general measurement operator.
Adjusting the notation a bit to match the form of the linear case, the MAP estimate of the Gaussian measure solves the problem
\begin{equation}
    \label{eqn:nonlinpinn}
    \hat{u} = \argmin_{u\in H_k}\frac{1}{2}\|\Gamma^{-1}(y-R(u))\|^2 + r\int_{\Omega}\frac{1}{2}\|\nabla u\|^2 + qu\:d\Omega,
\end{equation}
which follows from Proposition~\ref{prop:pinnsloss}.
If the data are collected according to eq.~(\ref{eqn:data}), then it is easy to verify that eq.~(\ref{eqn:nonlinpinn}) reduces to the original variational problem.

\section{Analysis}
\label{sec:analysis}

Having established the interpretation of the Brownian bridge as a physics-informed prior, we discuss some important properties of how the prior behaves.
Specifically, we state some results which may prove useful in scientific machine learning contexts, including regularity, finite-dimensional representations, and convergence in regression tasks.

\subsection{Regularity}
Much of the behavior of the prior in GP regression relies on the associated RKHS of the covariance kernel.
We will work in the situation where $\beta = 1$, as the results do not change for different $\beta \in (0,\infty)$.
\begin{lemma}
    \label{lem:bridgerkhs}
    The RKHS of eq.~(\ref{eqn:browniankernel}) is the space $H_k \coloneqq H^1_0(\Omega)= \{u\in H^1(\Omega) : u = 0,\:\mathrm{on}\:\partial\Omega\}$.
\end{lemma}
\begin{proof}
    The proof is in Appendix~\ref{apdx:B}.
\end{proof}

The above result provides us with an interesting method to prove the well-known result that Brownian bridge sample paths are nowhere differentiable.
This follows immediately by combining Lemma~\ref{lem:bridgerkhs} and Proposition~\ref{thm:GPsamples}.
\begin{corollary}
    Let $u$ be the Brownian bridge process. Then, $u$ is a.s. nowhere differentiable.
\end{corollary}

This fact may be viewed as undesirable in a machine learning context, especially in applications where the behavior of the sample paths are important.
An example of this could be an uncertainty propagation task, where samples from the posterior distribution are propagated through some other quantity of interest.
We would then like the samples to match the behavior of the ground truth to prevent unphysical predictions.

In what follows, we explore the possibility of redefining the GP so that samples match the behavior of the ground truth.
First, recall the next definition:
\begin{definition}[Version of a stochastic process~\cite{bremaud2014fourier}]
    Let $u$ be a stochastic process on $\Omega$.
    Then a stochastic process $\Tilde{u}$ on $\Omega$ is said to be a \emph{version} of $u$ if $u(x) = \Tilde{u}(x)$ a.s. for all $x \in \Omega$.
\end{definition}
We will look to find versions of the Brownian bridge on powers of its RKHS, also known as a Hilbert scale~\cite{mathe2008direct}:
\begin{definition}[Powers of RKHS~\cite{steinwart2008support}]
    Let $k:\Omega \times \Omega \to \mathbb{R}$ be a continuous, positive-definite kernel with RKHS $H_k$ and $(\lambda_n,\psi_n)_{n=1}^{\infty}$ be the eigensystem of the integral operator induced by $k$.
    Let $0<p \leq 1$ be a constant, and assume that $\sum_{n \in \N}\lambda_n^{p}\psi_n^2(x) < \infty$ holds for all $x\in\Omega$.
    Then the $p$-power of $H_k$ is the set
    $$
    H_k^{p} \coloneqq \left\{u \coloneqq \sum_{n=1}^{\infty} \alpha_n\lambda_n^{p/2}\psi_n: \sum_{n=1}^{\infty}\alpha_n^2<\infty\right\}.
    $$
    The inner product is $\langle u, v \rangle_{H_k^{p}} \coloneqq \sum \alpha_n\beta_n$ for $u  = \sum \alpha_n\lambda_n^{p/2}\psi_i$ and $v  = \sum \beta_n\lambda_n^{p/2}\psi_n$.
    Further, the $p$-power kernel of $k$ is the function $k^{p}(x,x') \coloneqq \sum_{n=1}^{\infty} \lambda
    _n^{p}\psi_n(x)\psi_n(x')$.
\end{definition}

Note we have the property $H_k = H_k^1 \subset H_k^{p_1} \subset H_{k}^{p_2} \subset L^2(\Omega)$, for all $0 < p_2 < p_1 < 1$.
Evidently, as $p$ decreases, the power RKHS loses some regularity.
Note that $H_k^p$ is itself a RKHS with kernel $k^p$.
Finally, we need the following theorem, which follows from Driscol's theorem~\cite[Theorem 3]{driscoll1973reproducing}.
\begin{theorem}[Theorem 4.12~\cite{kanagawa2018gaussian}]
    Let $k:\Omega \times \Omega \to \mathbb{R}$ be a continuous, positive-definite kernel with RKHS $H_k$, and $0 < p \leq 1$ be a constant.
    Assume $\sum_{n \in \N}\lambda_n^{p}\psi_n^2(x) < \infty$ holds for all $x\in\Omega$, where $(\lambda_n,\psi_n)_{n=1}^{\infty}$ is the eigensystem of the integral operator induced by $k$.
    Consider $u \sim \mathcal{GP}(0,k)$.
    Then, the following conditions are equivalent:
    \begin{enumerate}
        \item $\sum_{n\in\N} \lambda_n^{1-p} < \infty$.
        \item The natural injection $I_{kk^{p}}: H_k \to H_k^{p}$ is Hilbert-Schmidt.
        \item There exists a version $\Tilde{u}$ of $u$ with $\Tilde{u} \in H_{k}^{p}$ with probability one.
    \end{enumerate}
\end{theorem}

We can now prove the following.
\begin{proposition}
    \label{prop:versions}
    Let $u$ be the unit Brownian bridge with $d=1$.
    Then, for all $1/2 < p < 1$, there exists a version of $u$, $\Tilde{u}$, such that $\Tilde{u} \in H_k^p$ with probability one.
\end{proposition}
\begin{proof}
    First, we need to check when the condition $\sum_{n\in\N}\lambda_n^p\psi_n^2(x) < \infty$ for all $x\in \Omega$ holds.
    The eigenvalues and eigenfunctions are $(n^2\pi^2)^{-1}$ and $\phi_n(x) = \sqrt{2}\sin(n\pi x)$, $n\in \N$, respectively.
    Then for any $x \in \Omega$,
    $$
    \sum_{n\in\N} 2(n^2\pi^2)^{-p}\sin^2(n\pi x) \leq \sum_{n\in\N} 2(n^2\pi^2)^{-p} < \infty,
    $$
    when $1/2 < p  \leq 1$, which can be verified by the $p$-series test.
    We now will show (i) holds.
    We have $\sum_{n\in\N} \frac{1}{(n^2\pi^2)^{1-p}} < \infty$ for any $1/2 < p < 1$, which proves the result.
\end{proof}
The proposition shows that we can find a version of the Brownian bridge which is, in a sense, \emph{as close as possible} to being an $H_0^1(\Omega)$ function without being weakly differentiable.
If desired, one can construct these versions using the Karhunen-Lo\`eve expansion (KLE) of the $p$-power kernel.

While at first the poor regularity of the prior may feel discouraging, the fact that the sample paths a.s. do not belong to the solution space of the Poisson equation is less of an issue than seems.
In regressions tasks we often interpret $\Tilde{m}$ to be the predictor, with the variance representing a worst-case error, and the individual sample paths are inconsequential.
Recall by Proposition~\ref{prop:gpkrrequiv} that the posterior mean function $\Tilde{m}$ will live in the RKHS of the prior covariance.
As a result, the regularity of $\Tilde{m}$ will match the desired behavior required of the energy functional, i.e., an $H_0^1(\Omega)$ function, so in this sense, it is ideal that the RKHS is a first-order Sobolev space.
We do not add any additional smoothness assumptions beyond what is needed for the PDE solution to exist.
In fact, the RKHS being norm-equivalent to a Sobolev space is a crucial hypothesis needed to establish convergence conditions, explored later.

\subsection{Finite-dimensional representations}
\label{sec:finite}
We must work with a finite-dimensional representation of the prior in practical applications.
In most uses of GP regression a pointwise mesh of test points is placed on $\Omega$ where the posterior predictions are queried.
Instead, we derive a finite-dimensional basis approximation to the prior in $L^2(\Omega)$.
The motivation behind this is to enable scalable prior/posterior approximation algorithms by placing the computational burden on sampling basis function coefficients, rather than on the test points.
We prove convergence in measure and in Wasserstein distance under our approximation.
Without loss of generality, we derive the results with $\beta = 1$.
Since we do not permit $\beta$ to be zero or infinite, the results do not change for different values of $\beta$.

Our finite-dimensional representation is inspired by the mathematics behind quantum field theory, where Gaussian measures are oftentimes expressed under the path integral formalism.
Formally, by assuming the existence of the Lebesgue measure on $L^2(\Omega)$, we write the physics-informed prior $\mu_C\sim \mathcal{N}(0,C)$ as
\begin{equation}
    \label{eqn:physprior}
    \mu_C(du) ``=" \frac{1}{Z}\exp\left\{-\frac{1}{2}\langle u, Cu\rangle\right\}\mathcal{D}u,
\end{equation}
where we have centered the measure, and $Z = \int_{L^2(\Omega)}\exp\left\{-\frac{1}{2}\langle u, Cu\rangle\right\}\mathcal{D}u$ is the normalization constant.
The shift is justified since $u^0\in H_k$, which results in an equivalent measure.
Here, $\mathcal{D}u$ serves as a replacement for the non-existent Lebesgue measure in infinite-dimensions.
This idea appears in different path integral approaches for Bayesian inverse problems, including Bayesian field theory~\cite{lemm2003bayesian}, information field theory~\cite{ensslin2009information}, physics-informed information field theory~\cite{alberts2023physics,hao2024information}, and others~\cite{chang2014path}.
Of course, eq.~(\ref{eqn:physprior}) is not well-defined in the continuum limit, but, as the physicists do, we will look to extract meaning from this expression.
The reference~\cite{glimm2012quantum} provides a mathematical background to the nuances of using such definitions.

The formal Lebesgue density is useful for deriving finite-dimensional approximations to the prior measure.
In a finite-dimensional subset of $L^2(\Omega)$, eq.~(\ref{eqn:physprior}) is well-defined, which allows us to perform calculations.
Then, a limit procedure generates the correct Gaussian measure on $L^2(\Omega)$.
To this end, recall that a Borel cylinder set of a separable Hilbert space $H$ is a subset $I \subset H$ given by $I = \{u \in H : (\langle u, \psi_1\rangle, \dots, \langle u, \psi_n\rangle) \in A\}$, for $n\geq 1$, $\psi_1,\dots,\psi_n$ orthonormal, and $A$ a Borel subset of $\mathbb{R}^n$.
The collection of all cylinder sets is denoted by $\mathcal{R}$, and we let $\sigma(\mathcal{R})$ be the $\sigma$-algebra generated by $\mathcal{R}$.
One can show that $\sigma(\mathcal{R}) = \mathcal{B}(H)$, so it is sufficient to construct measures on cylinder sets.

Pick any orthonormal basis\footnote{One could choose a grid of piecewise constant functions on $\Omega$, which corresponds to picking test points.} in $L^2(\Omega)$, $(\psi_i)_{i\in\N}$, and let $\mathcal{F}_n \subset L^2(\Omega)$ be the set $\mathcal{F}_n = \{u\in L^2(\Omega): u = \sum_{i=1}^n a_i\psi_i, a_i\in \mathbb{R}\}$ for fixed $n\in\N$.
Then, $\dim(\mathcal{F}_n) = n < \infty$.
Let $\Sigma_{\mathcal{F}_n}$ be the restriction-corestriction of $C$ by both domain and codomain to $\mathcal{F}_n$.
In this case, $\Sigma_{\mathcal{F}_n}$ is the covariance matrix of a unique finite-dimensional Gaussian measure appearing as
\begin{equation}
    \label{eqn:finiteprior}
    \mu_n(d\hat{u}) = \frac{1}{\sqrt{(2\pi)^n|\Sigma_{\mathcal{F}_n}|}}\exp\left\{-\frac{1}{2}\langle \hat{u}, \Sigma_{\mathcal{F}_n}^{-1}\hat{u}\rangle\right\} d\hat{u},
\end{equation}
where the Lebesgue measure induced by the $L^2(\Omega)$-inner product, $d\hat{u}$, is well-defined.
So, eq.~(\ref{eqn:finiteprior}) can be regarded as a measure over a finite-dimensional space on the cylinder sets of $\mathcal{F}_n$.
The next series of results show that this measure has the correct limiting behavior.
\begin{proposition}
    \label{lemma:limit}
    Let $\mathcal{F}_n \subset \mathcal{F}_m \subset L^2(\Omega)$ with $\dim(\mathcal{F}_n) = n \leq \dim(\mathcal{F}_m) = m < \infty$ and $C$ be a covariance operator on $L^2(\Omega)$.
    Let the restriction-corestriction of $C$ to $\mathcal{F}_n$ and $\mathcal{F}_m$ be given by $\Sigma_{\mathcal{F}_n}$ and $\Sigma_{\mathcal{F}_m}$, respectively, and let $\mu_n$ and $\mu_m$ be as given in eq.~(\ref{eqn:finiteprior}) for each.
    Then, the restriction-corestriction of $\mu_m$ to $\mathcal{F}_n$ cylinder sets is exactly $\mu_n$.
\end{proposition}
\begin{proof}
    Note the $\mathcal{F}_n$ cylinder sets are also cylinder sets in $\mathcal{F}_m$.
    Both $\mu_C$ and $\mu_m$, when restricted to $\mathcal{F}_n$ cylinder sets, define Gaussian measures on $\mathcal{F}_n$, uniquely determined by their covariances.
    The measure $\mu_n$ has covariance $\Sigma_{\mathcal{F}_n}$ by definition, while restriction of $\mu_m$ to $\mathcal{F}_n$ cylinder sets has covariance
    given by the restriction-corestriction of $\Sigma_{\mathcal{F}_m}$ to $\mathcal{F}_n$, which is $\Sigma_{\mathcal{F}_n}$.
\end{proof}

The intuition behind the result is as follows.
In applications, we represent the prior as a truncated expansion corresponding to a set $\mathcal{F}_n \subset L^2(\Omega)$ so that we are in the finite-dimensional setting for sampling.
Adding additional terms, which corresponds to approximation in $\mathcal{F}_m,$ does not change how the prior behaves on $\mathcal{F}_n$, as $\mu_n$ and $\mu_m$ agree on the $\mathcal{F}_n$ cylinder sets.
Practically this means that in applications there is some cutoff point where refining the approximation any further does not reasonably change the results.

Once we have selected a finite-dimensional representation, a straightforward application of Bayes' rule reveals the posterior.
To keep the notation consistent with the infinite-dimensional setting, the resulting posterior, denoted $\mu^y_{n}$ can be described by the Radon-Nikodym derivative following Theorem~\ref{thm:bayes}.
That is, after describing the measurement process as a potential $\Phi$, we write
\begin{equation}
    \label{eqn:finitepost}
    \frac{d\mu^{y}_n}{d\mu_n}(\hat{u})\propto \exp\{-\Phi(\hat{u})\}, \quad \hat{u} \in \mathcal{F}_n.
\end{equation}
In what follows, we show that the finite-dimensional representation of the prior and the resulting posterior have the correct convergence behavior.
We begin with the 1D case, and later study the multidimensional case, which is much more delicate.

The first result is on the convergence of the prior.
\begin{theorem}
    \label{thm:convergence1}
    Let $d=1$, $(\psi_i)_{i\in\N}$ be an orthonormal basis for $L^2(\Omega)$, and for each $n\in\N$, let $\mu_n$ be given by eq.~(\ref{eqn:finiteprior}).
    Then $\mu_n\implies \mu_C$.
    That is, the sequence $(\mu_n)_{n\in\N}$ converges weakly to the Gaussian measure $\mu_C = \mathcal{N}(0,C)$ on $L^2(\Omega)$.
\end{theorem}
\begin{proof}
    We will show weak convergence in measure by showing convergence of characteristic functions.
    Choose any $u \in L^2(\Omega)$.
    For each $n$, the characteristic function of eq.~(\ref{eqn:finiteprior}), evaluated at $u$ is
    $$
    \phi_{\mu_{n}}(u) = \exp\left\{-\frac{1}{2}\left\langle \sum_{i=1}^na_i\psi_i, \Sigma_{\mathcal{F}}\left(\sum_{i=1}^na_i\psi_i\right)\right\rangle\right\},
    $$
    where $a_i = \langle u,\psi_i\rangle$.
    We have $\lim_{n\to\infty} \phi_{\mu_{n}}(u) = \exp\left\{-\frac{1}{2}\langle u, Cu\rangle\right\}$, which is the characteristic function of $\mathcal{N}(0,C)$~\cite[Lemma 2.1]{kuo2006gaussian}.
    Convergence in characteristic functions implies weak convergence in measure~\cite{billingsley2017probability}.
\end{proof}

It remains to show that posterior convergence also holds.
For this, we rely on~\cite{sprungk2020local}, which provides the result in terms of Wasserstein distance between measures.
That is, under some light assumptions about the potential $\Phi$, if $W_p(\mu_C, \mu_n) \to 0$ it follows that $W_p(\mu_C^y, \mu_n^y) \to 0$ for a given $p\geq 1$.
What we have shown thus far is that the sequence of priors $(\mu_n)_{n\in\N}$ converges weakly in measure to $\mu_C$.
However, it is not immediately obvious that this leads to $W_p$ convergence.
It turns out that this does hold provided that the $p$-th moments of $\mu_n$ also converge to the $p$-th moment of $\mu_C$, see~\cite[Theorem 6.9]{villani2008optimal}.
This leads us to the following lemma.
\begin{lemma}
    \label{lem:momentconv}
    Let the conditions of Theorem~\ref{thm:convergence1} hold, $X_n \sim \mu_n$, and $X \sim \mu_C$.
    Then for every $p \geq 1$, in the limit $n \to \infty,$ we have $\mathbb{E}[\|X_n\|^p] \to \mathbb{E}[\|X\|^p]$.
\end{lemma}
\begin{proof}
    Let $p\geq 1$ be fixed.
    Our choice of a truncated basis defines a projection mapping from $L^2(\Omega)$ onto $\mathcal{F}_n$ for any $n$, which we call $P_n$.
    Then, for every $n$, the random variable $X_n$ is given by $X_n = P_nX$.
    We will show that $\|X_n\|^p \to \|X\|^p$, which will give the desired moment convergence.
    
    To start, note that $P_n \to I$ (the identity operator on $L^2(\Omega)$), so for all $u \in L^2(\Omega)$, we have $P_n u \to u$.
    Application with the random element $X$ yields
    $$
    \|X_n-X\| = \|P_nX - X\| \to 0, \quad \text{a.s.},
    $$
    and by continuity of the norm, this gives $\|X_n\|^p \to \|X\|^p$ with probability $1$.
    Also, since $\|P_nu\|\leq \|u\|$ $\forall u\in L^2(\Omega)$, it follows that $\|X_n\|^p \leq \|X\|^p$ a.s. for any $n$.

    Now, as $X$ is Gaussian, the Fernique theorem applies, which states that for some $\alpha > 0$, $\mathbb{E}[\exp (\alpha \|X\|^2)] < \infty$~\cite[Theorem 6.9]{stuart2010inverse}.
    In particular, this shows that $\mathbb{E}[\|X\|^p]<\infty$.
    Hence, by the dominated convergence theorem, we find
    $$
    \mathbb{E}[\|X_n\|^p] \to \mathbb{E}[\|X\|^p],
    $$
    which shows the desired convergence in moments.
\end{proof}

We can now prove convergence of the posterior approximation.
\begin{theorem}
    \label{thm:posterconv1d}
    Let $d = 1$, $(\psi_i)_{i\in\N}$ be an orthonormal basis of $L^2(\Omega)$, and $\mu_C = \mathcal{N}(0,C)$ with corresponding posterior $\mu_C^y$ as given by Theorem~\ref{thm:bayes}.
    For each $n\in\N$, let $\mu_n$ be given by eq.~(\ref{eqn:finiteprior}) with posterior $\mu_n^y$ eq,~(\ref{eqn:finitepost}).
    Further, assume that $\Phi:L^2(\Omega)\to [0,\infty)$ is continuous and measurable.
    Then for any $p\geq 1$, we have $W_p(\mu_C^y,\mu_n^y)\to 0$ as $n \to \infty$.
\end{theorem}
\begin{proof}
    The proof is a simple application of~\cite[Lemma 16]{sprungk2020local}, which states that $W_p$ convergence of the priors implies convergence in the posteriors.
    Combining Lemma~\ref{lem:momentconv} with~\cite[Theorem 6.9]{villani2008optimal}, we have $W_p(\mu_C,\mu_n)\to 0$, which gives the result.
\end{proof}

Next, we move to the general case $d>1$.
Recall that in this setting the Brownian bridge does not define a Gaussian measure on $L^2(\Omega)$ due to the fact that the associated covariance is not a trace-class operator under the $L^2$-norm.
However, we can still provide parallel convergence theorems as with the case $d=1$ if we adjust the function spaces involved.
One option is to work in the setting of tempered distributions, which is discussed in Appendix~\ref{apdx:C}.
This is common in white noise analysis~\cite{hida2013white}, e.g. stochastic PDEs~\cite{da2014stochastic}, and mathematical physics~\cite{glimm2012quantum}.
Instead, we will define the prior on the dual of an appropriate Sobolev space, otherwise known as a negative Sobolev space.

Recall that for a $\tau \in \mathbb{R}$, the dual of $H^{\tau}(\Omega)$ is given by
$$
\mathcal{H}^{-\tau}(\Omega) = \left\{u: H^{\tau}(\Omega) \to \mathbb{R}:  \|u\|^2_{H^{-\tau}} \coloneqq \sum_{|\alpha|\in \N}\lambda_{\alpha}^{\tau}\langle u,\psi_{\alpha}\rangle^2 < \infty \right\},
$$
where again $\phi_{\alpha}$ and $\lambda_{\alpha}$ are the eigenfunctions with corresponding eigenvalues of the Brownian bridge, as given  by eq.~(\ref{eqn:efuncs}) and eq.~(\ref{eqn:evals}), respectively~\cite{adams2003sobolev}.
First, we show that the covariance of the Brownian bridge has finite trace on $H^{-\tau}(\Omega)$ when $\tau > d/2-1$.

\begin{proposition}
    \label{prop:dtrace}
    Let $C$ be the covariance operator of the Brownian bridge with $d > 1$.
    Then $C$ defines a trace class operator on $H^{-\tau}(\Omega)$ if and only if $\tau > d/2 - 1$.
\end{proposition}
\begin{proof}
    On the space $H^{-\tau}(\Omega)$, the trace class norm is $\|C\|_{\tr{H^{-\tau}}} = \sum_{|\alpha|\in \N}\langle C\psi_{\alpha}, \psi_{\alpha}\rangle_{H^{-\tau}}.$ Choosing the basis to be the eigenfunctions of $C$, an individual entry of the series is
    \begin{align*}
        c_{\alpha} &= \langle C\psi_{\alpha}, \psi_{\alpha}\rangle_{H^{-\tau}} \\
        &= \langle \lambda_{\alpha}\psi_{\alpha}, \psi_{\alpha}\rangle_{H^{-\tau}} \\
        &= \lambda_{\alpha}\|\psi_{\alpha}\|^2_{H^{-\tau}} \\
        &= \lambda_{\alpha}^{1 + \tau} \langle \psi_{\alpha}, \psi_{\alpha}\rangle^2 = \lambda_{\alpha}^{1+ \tau}.
    \end{align*}
    Plugging in the form of the eigenvalues, we have $\|C\|_{\tr{H^{-\tau}}} = \sum_{|\alpha|\in\N}(\pi|\alpha|)^{-2(1 + \tau)}$, which converges if and only if $2(1 + \tau) > d$ by comparison to an integral.
\end{proof}

Proposition~\ref{prop:dtrace} shows that the physics-informed prior is a well-defined Gaussian measure on $H^{-\tau}(\Omega)$ for $\tau$ sufficiently large, and we do not need to restrict ourselves to the setting of GPs.
This means that we are justified in writing $\mu_C \sim \mathcal{N}(0, C)$.
Unfortunately, these spaces are a bit larger than we might like, given that $H^{-\tau}(\Omega)$ contains distributions.
We also have the inclusion relation $H^{\tau}(\Omega) \subset L^2(\Omega) \subset H^{-\tau}(\Omega)$.
In a machine learning context, it is difficult to make sense of generating such a sample.
Fortunately, the Brownian bridge is known to produce continuous sample paths on the unit cube~\cite{adler2010geometry}.

There are certain advantages gained by working in this space, however.
Namely that, since the prior is indeed a probability measure, we gain access to both Bayes's theorem and the Wasserstein distance.
In order to apply Bayes's theorem, we must ensure that the domain of the observation map $R$ is a space with full measure under $\mu_{C}$.
This holds when $R$ takes continuous functions as input, which we have assumed in view of Proposition~\ref{prop:regularity}.
Hence, by Bayes's rule we identify a posterior measure $\mu^y$ and in similar fashion a finite-dimensional approximation $\mu^y_n$, both viewed as probability measures over $H^{-\tau}(\Omega)$.
With that out of the way, we provide convergence criteria, starting with the prior.
The first step is to show convergence of the covariance forms.
\begin{lemma}
    \label{lem:covconv}
    Let $d > 1$, $\tau > d/2 -1$, $(\psi_{\alpha})_{|\alpha|\in\N}$ be the eigenfunctions of $C$, and for each $n$, $P_n$ be the projection of $H^{-\tau}(\Omega)$ onto the finite span $\{\psi_{\alpha}\}_{|\alpha|=1}^{n}$, i.e. $\mathcal{F}_n = P_n(H^{-\tau}(\Omega))$.
    Then the finite rank covariances $\Sigma_{\mathcal{F}_n} = P_nCP_n$ satisfy $\|C-\Sigma_{\mathcal{F}_n}\|_{\tr{H^{-\tau}}} \to 0$.
\end{lemma}
\begin{proof}
    Beyond the order of the projection, $n$, the operators $C$ and $C - \Sigma_{\mathcal{F}_n}$ have the same eigenvalues.
    So, $\|C - \Sigma_{\mathcal{F}_n}\|_{\tr{H^{-\tau}}} = \sum_{|\alpha|>n}(\pi|\alpha|)^{-2(1 + \tau)}$, which is just the tail of the same convergent series from the proof of Proposition~\ref{prop:dtrace}.
\end{proof}
This shows that by projection onto the eigenfunctions, the finite-dimensional approximation we have chosen can also be viewed as a Gaussian measure on $H^{-\tau}$.
Noting that Gaussian measures are uniquely determined by their covariance forms, Lemma~\ref{lem:covconv} also shows that $\mu_n$ converges in some sense to the correct measure.
In the next result, we make this precise in terms of $W_2$ distance.
\begin{theorem}
    \label{thm:generalpostconv}
    Let $d > 1$, $\tau > d/2 -1$, and for each $n\in\N$, $P_n$ be the projection onto the finite span of eigenfunctions $\{\psi_{\alpha}\}_{|\alpha|=1}^n$.
    Further, let $\mu_n$ be given by eq.~(\ref{eqn:finiteprior}) with $\mathcal{F}_n = P_n(H^{-\tau}(\Omega))$ with corresponding posterior $\mu^y_n$ following eq.~(\ref{eqn:finitepost}).
    Then in the limit $n \to \infty$, it holds that both $W_2(\mu_C,\mu_n) \to 0$ and $W_2(\mu^y,\mu^y_n)\to 0$.
\end{theorem}
\begin{proof}
    The result follows from application of Gelbrich's formula to the priors
    $$
    W^2_2(\mu_C,\mu_n) = \|C\|_{\tr{H^{-\tau}}} + \|\Sigma_{\mathcal{F}_n}\|_{\tr{H^{-\tau}}} - 2\left\|\sqrt{C^{1/2}\Sigma_{\mathcal{F}_n}C^{1/2}}\right\|_{\tr{H^{-\tau}}}.
    $$
    It holds from continuity of the map $A \mapsto A^{1/2}$, continuity of the norm, and by Lemma~\ref{lem:covconv} that $\|\Sigma_{\mathcal{F}_n}\|_{\tr{H^{-\tau}}} \to \|C\|_{\tr{H^{-\tau}}}$ and
    $$
    \left\|\sqrt{C^{1/2}\Sigma_{\mathcal{F}_n}C^{1/2}}\right\|_{\tr{H^{-\tau}}} \to \|C\|_{\tr{H^{-\tau}}},
    $$
    which can be checked by spectral representation of the operators.
    Therefore in the limit $W_2(\mu_C, \mu_n) \to 0$.
    Posterior convergence follows immediately from~\cite[Lemma 16]{sprungk2020local}.
\end{proof}
Of course, an immediate corollary is that $\mu_n^y \implies \mu^y$ by~\cite[Theorem 6.9]{villani2008optimal}.

Theorems~\ref{thm:posterconv1d} and~\ref{thm:generalpostconv} justify the use of eq.~(\ref{eqn:finiteprior}) in applications.
In 1D, the finite-dimensional representation of the prior converges to the correct Gaussian measure on $L^2(\Omega)$.
We can use this form to derive additional results in the next section.
In the multidimensional case, eq.~(\ref{eqn:finiteprior}) converges to the correct measure on the dual of a Sobolev space.

\subsection{Convergence properties}

We now discuss conditions for which the posterior converges to the ground truth and in what sense.
Thankfully, understanding the convergence behavior is fairly straightforward due to the work of~\cite{teckentrup2020convergence}.
By applying the theorems derived in that work, we can prove that the posterior mean function will converge to the ground truth in the limit of infinite observations.
This holds even if we estimate the hyperparameters of the prior.
This fact is very relevant for us, since the source term $q$ could be treated as an unknown hyperparameter to the physics-informed prior.
Again, we will start with the case $d=1$ to illustrate.

To begin, we must discuss a bit about how the data should be collected in order for the convergence conditions to hold.
First, we restrict ourselves to point measurements in the domain $\Omega$.
Then, we must characterize how uniformly the data points are collected in the domain.
Let the set $X_n = (x_1,\dots,x_n) \subset \Omega$ represent the points at which the measurements are collected.
The \emph{fill distance} is defined by
$$
h_{X_n} \coloneqq \sup_{x \in \Omega}\inf_{x_i\in X_n}\|x-x_i\|,
$$
which measures the maximum distance any $x\in \Omega$ can be from $x_i \in X_n$.
The \emph{separation radius} is given by
$$
r_{X_n} \coloneqq \frac{1}{2}\min_{i \neq j} \|x_i-x_j\|,
$$
which measures half the minimum distance between any two different data collection points.
Lastly, the \emph{mesh ratio} is
$$
\rho_{X_n} \coloneqq \frac{h_{X_n}}{r_{X_n}} \geq 1.
$$
Both $h_{X_n}$ and $r_{X_n}$ go to $0$ as $n\to \infty$ under a space-filling design, for example the uniform grid or Sobol net~\cite{niederreiter1992random, wynne2021convergence}.
In what follows, we will assume the measurements are collected on a uniform grid, so that $\rho_{X_n}$ is constant with $n$ and the calculations are simplified.
The theorems also hold for any data collection scheme where $\rho_{X_n}$ is bounded above.

We rely on the two main convergence theorems from~\cite{teckentrup2020convergence}.
Notably, the results are concerned with the case where the GP prior contains unknown hyperparameters which are approximated along with the field.
In that work, empirical Bayes is taken as the motivating example.
The conditions on the hyperparameters are fairly loose and the convergence theorems will hold in a wide variety of cases.
The way in which the hyperparameters are learned does not impact the results, and one may prefer an alternative such as a maximum likelihood estimate (MLE).
We put a specific focus on the MAP estimate/evidence approximation in Section~\ref{sec:model}.

In our case, the unknowns could be the source term, $q$, or the parameter $\beta$.
It is standard to parameterize $q$ with $\hat{q}(\cdot;\theta)$, so that the inverse problem is no longer infinite-dimensional.
For example, we may represent $q$ as a polynomial, truncated basis, or as a neural network.
Alternatively, we may model $q$ with a truncated KLE to incorporate prior knowledge under the fully Bayesian treatment.
Then, the parameters $\theta$ (along with $\beta$) enter the physics-informed prior as hyperparameters which may be tuned.
Let the vector $\lambda = (\theta, \beta)$ be the list of all hyperparameters.
The first theorem is a condition on the convergence of the posterior mean function, $\Tilde{m}(\cdot;\lambda)$, to the ground truth, $u^*$.

\begin{theorem}[Theorem 3.5~\cite{teckentrup2020convergence}]
    \label{thm:convergence}
    Let $(\hat{\lambda}_i)_{i=1}^\infty \subseteq \Lambda$ be a sequence of estimates for $\lambda$ with $\Lambda \subseteq \mathrm{dom}(\lambda)$ compact.
    Assume the following hold:
    \begin{enumerate}
        \item $\Omega$ is compact with Lipschitz boundary for which an interior cone condition holds.
        \item The RKHS of $k(\cdot,\cdot;\lambda)$ is isomorphic to the Sobolev space $H^{\tau(\lambda)}(\Omega)$ for some $\tau(\lambda) \in \mathbb{N}$.
        \item $u^* \in H^{\bar{\tau}}(\Omega)$, for some $\bar{\tau} = \alpha + \gamma$ with $\alpha \in \mathbb{N}$, $\alpha > d/2$, and $0\leq \gamma < 1$.
        \item $u^0(\cdot;\lambda) \in H^{\bar{\tau}}(\Omega)$ for each $\lambda \in \Lambda$.
        \item For some $N^* \in \mathbb{N}$, the quantities $\tau^- = \inf_{n\geq N^*}\tau(\hat{\lambda}_n)$ and $\tau^+ = \sup_{n\geq N^*}\tau(\hat{\lambda}_n)$ satisfy $\Tilde{\tau} = \alpha' + \gamma'$ with $\alpha' \in \mathbb{N}$, $\alpha'>d/2$ and $0\leq \gamma'<1$.
    \end{enumerate}
    Then there exists a constant $c$, independent of $u^*, u^0$, and $n$, such that for any $p \leq \bar{\tau}$,
    \begin{equation}
        \label{eqn:convcond}
        \left\|u^* - \Tilde{m}(\cdot;\hat{\lambda}_n)\right\|_{H^p(\Omega)} \leq ch_{X_n}^{\min{\bar{\tau},\tau^-}-p}\rho_{X_n}^{\max{\tau^+-\bar{\tau},0}}\left(\|u^*\|_{H^{\bar{\tau}}(\Omega)}+\sup_{n\geq N^*}\|u^0(\cdot;\hat{\lambda}_n)\|_{H^{\bar{\tau}}(\Omega)}\right),
    \end{equation}
    for all $n\geq N^*$ and $h_{X_n} \leq h_0$.
\end{theorem}

We discuss some of the assumptions of Theorem~\ref{thm:convergence} in the context of our problem.
The third assumption is a regularity constraint on the ground truth.
Since we are mostly concerned with identifying the solution to the Poisson equation, this is reasonable to impose.
Assuming that $u^*$ is a solution to the Poisson equation (for sufficiently regular domain and source term), we would expect at the minimum $u^* \in H^2(\Omega)$ in light of Proposition~\ref{prop:regularity}, which satisfies (iii) up to $d = 3$, e.g. picking $\gamma = 0.5$ when $d=3$.
Observe that we may have convergence for \emph{any} sufficiently smooth ground truth field, not just solutions to the assumed PDE.
This is relevant, for instance, in the case of model-misspecification.
This could result from an incorrectly identified source or perhaps $u^*$ is better modeled by the nonlinear Poisson equation, among others.

Assumption (iv) is a regularity constraint on the prior mean function, $u^0$.
As with assumption (iii), this is easy to satisfy, as $u^0(\cdot;\lambda)$ is exactly a solution to the Poisson equation for any $\lambda$.
As an example, if we represent $q$ as a neural network with a smooth activation function or as a GP with smooth sample paths, then this assumption trivially holds, even as the network weights are updated.

The final assumption is related to how the hyperparameters are learned.
The quantities $\tau^-$ and $\tau^+$ are essentially $\lim\inf \tau(\hat{\lambda}_n)$ and $\lim\sup \tau(\hat{\lambda}_n)$.
This assumption simply requires the the RKHS of the prior covariance to remain sufficiently smooth as the hyperparameters are optimized, and immediately holds if $\lambda$ is kept fixed.
In our physics-informed prior, $q$ does not enter the covariance, so this is only an assumption on $\beta$.
Restricting $\beta$ to a compact interval will satisfy this condition, as the RKHS does not change as $\beta$ moves.
We encourage the reader to refer to~\cite{teckentrup2020convergence} for details on optimal convergence rates.

With this out of the way, we can prove the following convergence theorem.

\begin{theorem}[Convergence of Brownian bridge GP]
    \label{thm:brownmeanconv}
    Let $\Omega = [0,1]$, $q(\cdot;\theta) \in L^2(\Omega)$ for all $\theta$, $u^* \in H^2(\Omega)$, and $u^0(\cdot;\theta)$ be the solution to eq.~(\ref{eqn:poisson}).
    Take $\hat{\lambda}_n \subset \Lambda$ to be a sequence of estimates for the collection $(\theta,\beta)$ for compact $\Lambda \subseteq \mathrm{dom}(\lambda)$.
    Then the GP posterior mean function, $\Tilde{m}(\cdot;\hat{\lambda}_n)$, given by eq.~(\ref{eqn:GPmean}), with prior $u(\cdot;\hat{\lambda}_n)\sim \mathcal{GP}(u^0(\cdot;\hat{\theta}_n), (-\hat{\beta}_n\Delta)^{-1})$, converges in $L^2(\Omega)$ to $u^*$ in the limit of infinite observations.
    That is,
    $$
    \lim_{h_{X_n}\to 0}\|u^* - \Tilde{m}(\cdot ; \hat{\lambda}_n)\|_{L^2(\Omega)} = 0.
    $$
\end{theorem}
\begin{proof}
    We verify the assumptions of Theorem~\ref{thm:convergence} one by one.
    $\Omega = [0,1]$ trivially satisfies (i).
    By Lemma~\ref{lem:bridgerkhs}, we have that the RKHS of $k(\cdot,\cdot;\lambda) = (-\hat{\beta}\Delta)^{-1}$ is norm-equivalent to $H^1(\Omega)$ for any $0<\beta<\infty$, which satisfies (ii) with $\tau = 1$.
    Assumption (iii) holds by choosing $\alpha = 3$, $\gamma = 0.5$.
    Since $q \in L^2(\Omega)$, $u^0(\cdot;\lambda) \in H^1(\Omega) \cap H^2(\Omega)$ for all $\lambda$ by the regularity of the Poisson equation, and (iv) holds.
    The assumptions on $\hat{\lambda}_n$ were chosen to satisfy (v) with $\alpha' = 3$, $\gamma = 0.5$.
    Finally, the inequality eq.~(\ref{eqn:convcond}) gives $\|u^* - \Tilde{m}(\cdot;\hat{\lambda}_n)\|_{H^2(\Omega)} \to 0$ as $h_{X_n} \to 0$, and application of the Sobolev embedding theorem yields convergence in $L^2(\Omega)$-norm.
\end{proof}

An immediate corollary is the following, which we state in terms of the variational problem we started with.
\begin{corollary}
    \label{cor:conv}
    Let $E(u;\theta) = \int_{\Omega} \frac{1}{2}\|\nabla u\|^2 + q(\cdot;\theta)u\:d\Omega$, $\hat{r}_n = \sigma^2\hat{\beta_n}/n$, and the assumptions of Theorem~\ref{thm:brownmeanconv} hold.
    Then,
    $$
    \lim_{h_{X_n}\to 0}\|u^* - \hat{u}(\cdot;\hat{\lambda}_n)\|_{L^2(\Omega)} = 0,
    $$
    where $\hat{u}(\cdot,\hat{\lambda}_n)$ is the solution to the physics-regularized inverse problem
    $$
    \hat{u}(\cdot;\hat{\lambda}_n) = \argmin_{u\in H_k}\mathcal{L}_{\mathrm{data}}(u) + \hat{r}_nE(u;\hat{\theta}_n).
    $$
\end{corollary}
\begin{proof}
    By Theorem~\ref{thm:brownmeanconv}, the limit holds for $\Tilde{m}(\cdot;\hat{\lambda}_n)$, and by Theorem~\ref{thm:iftpinnequiv}, we have $\Tilde{m}(\cdot;\hat{\lambda}_n) = \hat{u}(\cdot;\hat{\lambda}_n)$.
\end{proof}

Under the same assumptions of Theorem~\ref{thm:brownmeanconv}, we can also prove that the posterior variance converges to zero.
\begin{theorem}[Collapse of Brownian bridge GP variance]
    \label{thm:browncovconv}
    Let all assumptions of Theorem~\ref{thm:brownmeanconv} hold.
    Then
    $$
    \lim_{h_{X_n}\to 0}\|\Tilde{k}^{1/2}(\cdot,\cdot;\hat{\lambda}_n)\|_{L^2(\Omega)} = 0,
    $$
    where $\Tilde{k}(\cdot,\cdot;\hat{\lambda}_n)$ is the posterior covariance function, given by eq.~(\ref{eqn:GPcov}), trained with prior
    $$
    u(\cdot;\hat{\lambda}_n) \sim \mathcal{GP}(u^0(\cdot;\hat{\theta}_n),(-\hat{\beta}_n\Delta)^{-1}),
    $$
    evaluated at $x=x'$.
\end{theorem}
\begin{proof}
    The hypotheses of Theorem~\ref{thm:convergence} are the exactly the same as what is found in~\cite[Theorem 3.8]{teckentrup2020convergence}, which shows there exists a constant $c$, independent of $n$, with
    $$
    \|\Tilde{k}^{1/2}(\cdot,\cdot;\hat{\lambda}_n)\|_{L^2(\Omega)} \leq ch_{X_n}^{\min{(\bar{\tau},\tau_-)}-d/2-\varepsilon}\rho_{X_n}^{\max{(\tau^+-\bar{\tau},0)}},
    $$
    for each $n \geq N^*$, $h_{X_n}\leq h_0$, and $\varepsilon>0$.
    Letting $h_{X_n}\to 0$ proves the result. 
\end{proof}

\begin{remark}
    The above results are valid for the case $d=1$.
    Similar convergence theorems also hold for $d>1$, but one must instead rely on~\cite[Theorem 3.11]{teckentrup2020convergence}, which exploits the tensor product structure of the covariance kernel.
    Note that in order for the results to hold for $d>1$, a sparse grid data collection scheme must be used.
\end{remark}

We now mention some implications of Theorem~\ref{thm:brownmeanconv}.
The first observation is that convergence holds even under significant model-form error.
In practice we a priori assume the ground truth satisfies the Poisson equation.
If we have selected the wrong model, i.e., the Poisson equation does not model the system accurately, then convergence still holds provided that the ground truth satisfies some smoothness constraints.
The same is true for model-form error resulting from picking the wrong source term or incorrectly identifying $q$ if we are solving the inverse problem.

The assumptions on $q$ and $\beta$ are rather loose in the application of this theorem.
When solving the inverse problem, the conditions on Theorem~\ref{thm:brownmeanconv} may be satisfied even if we have identified a bad estimate for $q$.
In fact, $q$ need not be identifiable.
The main requirement is that $\lambda$ remains in a compact domain.
If we represent $q$ with a neural network, this is satisfied if we do not allow the weights to explode.
Unfortunately, we are unable to prove if an estimate of $q$ will also converge to the ground truth.
We leave the discussion on this to Section~\ref{sec:model}.

\subsection{A note on the use of neural networks}

Our finite-dimensional representation was derived using a truncated orthonormal basis.
It is natural to ask whether other parameterizations are also suitable, particularly for deep neural networks.
For instance, this would provide a basis to connect our work to PINNs (more specifically deep Ritz).
When working with neural networks, there are some technical issues which must be treated with care.
We touch on both treating the the space $\mathcal{F}$ as a collection of neural networks as well as the convergence theorems.

Recall we look to approximate the prior with the finite-dimensional representation as given by eq.~(\ref{eqn:finiteprior}).
This representation is not immediately well-defined if we parameterize $u$ with a neural network.
To summarize, for a fixed neural network structure, we cannot assume that the space of functions the network can represent is finite-dimensional, in which case the Lebesgue measure would not exist.
To explain this, we introduce some notation following~\cite{petersen2021topological, mahan2021nonclosedness}.

Let $\Phi = \{(A_{\ell},b_{\ell})\}_{\ell = 1}^{n_L}$ be a set of matrix-vector tuples where $A_{\ell}\in \mathbb{R}^{N_{\ell} \times N_{\ell - 1}}$ and $b_{\ell} \in \mathbb{R}^{N_{\ell}}$ for each $\ell$.
The architecture of the network is given by $S = (N_0,N_1,\dots,N_{n_L})$, where $N(S)$ is the total number of neurons and $n_L = n_L(S)$ is the number of layers.
The collection $\Phi$ represents the values of the weights for a neural network with architecture $S$.
Then, for an activation function $h: \mathbb{R}\to \mathbb{R}$, the neural network is given by a mapping $\mathrm{NN}_{h}(\Phi): \Omega \to \mathbb{R}$.
We are interested in the properties of the function space induced by the network for fixed $S$ and $h$.
We will denote this set by $\mathcal{R}(\mathrm{NN_{h}})(S)$.

As it turns out, if we allow $\Phi$ to vary arbitrarily, then $\mathcal{R}(\mathrm{NN_{h}})(S)$ is not closed in $L^p(\Omega)$, $0< p< \infty$, for all activation functions commonly used in PINNs~\cite[Theorem 3.1]{petersen2021topological}.
The same is true in Sobolev spaces~\cite{mahan2021nonclosedness}.
This implies that the function space $\mathcal{F} = \mathcal{R}(\mathrm{NN_{h}})(S)$ is not finite-dimensional and eq.~(\ref{eqn:finiteprior}) is no longer well-defined.

However, if $\Phi$ is restricted to a compact set, then $\mathcal{R}(\mathrm{NN_{h}})(S)$ is compact in $L^p(\Omega)$~\cite[Proposition 3.5]{petersen2021topological}.
This compact restriction of $\Phi$ results from schemes which prevent exploding weights, a common practice.
While the result is nicer, it is still not immediately applicable to the construction of our finite-dimensional approximation: there is no guarantee that a compact set of a Hilbert space will be finite-dimensional.
The Hilbert cube is one example.
Although, we may approximate any compact set with a finite-dimensional subspace to arbitrary accuracy.
\begin{theorem}[\cite{stan2000paley}]
    Let $H$ be a Hilbert space.
    A subset $K \subset H$ is compact if and only if $K$ is closed, bounded, and for any $\varepsilon > 0$, there exists a finite-dimensional subspace $\mathcal{F}\subset H$ such that $\forall u \in K$, $\inf_{v \in \mathcal{F}}\|u-v\| < \varepsilon$.
\end{theorem}
Therefore, one could theoretically take $\mathcal{F}$ to be a finite-dimensional space which approximates $\mathcal{R}(\mathrm{NN}_h)(S)$ to a given tolerance, $\varepsilon$.
The size of $\mathcal{F}$ on which $\mu_n$ is defined may be adjusted by tweaking $\varepsilon$, changing the bound on the weights, or changing the network structure.

In Corollary~\ref{cor:conv}, we show that the function which solves the physics-regularized inverse problem will converge to the ground truth in the large-data limit.
This is if we solve the problem in the \emph{infinite-dimensional} setting.
Ideally we would like to derive the result for training neural networks.
The solution to this problem will be a function which lives in the RKHS $H^1_0(\Omega)$.
Again if we allow $\Phi$ to vary arbitrarily, then $\mathcal{R}(\mathrm{NN}_h)(S)$ is not closed in $H^1_0(\Omega)$.
This means that there are functions in $H^1_0(\Omega)$ for which the neural network must send $\|\Phi\|\to\infty$ in order to approximate.
If the ground truth happens to be such a function, then the convergence theorem will not hold.
Likewise, if we limit $\Phi$ to a compact set, then $\mathcal{R}(\mathrm{NN}_h)(S)$ is compact in $H^1_0(\Omega)$.
In this case, the neural network is only able to approximate a function to any accuracy if that function is also a neural network, so it is unlikely that convergence holds.
The only case where convergence to the ground truth could hold is if we allow the architecture of the neural network to change arbitrarily so that we may rely on the universal approximation theorem~\cite{hornik1989multilayer}.
However, establishing this connection is beyond the scope of our current work.

\section{On model-form error}
\label{sec:model}

In this section we perform an analysis of the hyperparameter $\beta$ towards the application of detecting model-form error.
Since $\beta$ is a hyperparameter of the GP prior, it is natural to assess how $\beta$ is learned during training.
We show the optimal choice of $\beta$ adjusts according to model-misspecification.
We build towards the result by working with the finite-dimensional distributions discussed in Section~\ref{sec:finite}.

Start by introducing a finite-dimensional representation of $L^2(\Omega)$.
This representation induces the function space $\mathcal{F} \subset L^2(\Omega)$ with $\dim(\mathcal{F}) = M < \infty$.
We will then study the posterior behavior of $\beta$ in the limit of infinite data.
Given our training data of the form eq.~(\ref{eqn:data}), we begin by writing the problem down as a hierarchical model:
\begin{align}
    \beta &\sim p(\beta), \nonumber \\
    \hat{u} | \beta &\sim \mathcal{N}(\hat{u}^0,\beta^{-1}\Sigma_{\mathcal{F}_M}), \nonumber\\
    d | \hat{u} &\sim \mathcal{N}(\hat{u},\sigma^2I), \label{eqn:hiearchy}
\end{align}
where $\hat{u}^0$ is the projection of $u^0$ onto $\mathcal{F}$, and $\Sigma_{\mathcal{F}_M}$ is the restriction of the covariance operator given by eq.~(\ref{eqn:browniankernel}) to $\mathcal{F}$.
Since we are no longer in the infinite-dimensional setting, application of Bayes's rule in the usual sense is justified, and we can also rely on the Lebesgue integral when deriving expressions.
We derive the joint posterior
$$
p(\hat{u}, \beta | d) = \frac{1}{Z}p(d | \hat{u})p(\hat{u}|\beta)p(\beta),
$$
where $Z$ is the unknown normalization constant.

To identify a deterministic estimate of $\beta$, we look to identify the MAP estimate
$$
\beta^* = \argmax_{\beta \in (0,\infty)} \log\int\frac{1}{Z}p(d | \hat{u})p(\hat{u}|\beta)d\hat{u} + \log p(\beta).
$$
Note that this is equivalent to maximizing the $\log$-evidence given by $\mathcal{L}(\beta) \coloneqq \log p(d|\beta)+\log p(\beta)$.
In what follows, we will show the MAP estimate is unique in the limit of large data, for certain choices of $p(\beta)$.
We start by deriving an expression for the gradient of the target function.
Throughout, we will center the space so that the prior mean function becomes $0$.
We have shown this is valid as the prior mean function does not depend on $\beta$ and it also lives in $H_k$.
\begin{lemma}
    \label{lem:grad}
    Consider the hierarchical model eq.~(\ref{eqn:hiearchy}).
    The $\log$-evidence of this model $\mathcal{L}(\beta) = \log p(d|\beta)+\log p(\beta)$ satisfies
    \begin{equation*}
        \label{eqn:lem6p1}
        \frac{\partial}{\partial\beta} \mathcal{L}(\beta) = \frac{M}{2\beta} + \frac{\partial}{\partial\beta}\log p(\beta) - \frac{1}{2}\langle\Tilde{m}-\hat{u}^0,\Sigma^{-1}_{\mathcal{F}_M}(\Tilde{m}-\hat{u}^0)\rangle - \frac{1}{2}\tr{\Sigma^{-1}_{\mathcal{F}_M}\Tilde{\Sigma}_{\mathcal{F}_M}},
    \end{equation*}
    where $\Tilde{m}(\cdot)$ and $\Tilde{\Sigma}_{\mathcal{F}_M}:\mathcal{F \to \mathcal{F}}$ are given by the posterior mean function eq.~(\ref{eqn:GPmean}) and posterior covariance form eq.~(\ref{eqn:GPcov}), respectively, and $\hat{u}^0$ is the projection of $u^0$ onto $\mathcal{F}$.
\end{lemma}
\begin{proof}
    The $\log \beta$ term is obvious.
    What remains is the marginal $\log$-likelihood term $\log p(d|\beta)$.
    We start by using Fisher's identity, see~\cite[eq. 3.1]{louis1982finding} or~\cite[Prop. 10.1.6]{cappe2005inference}, which relates the derivative of the marginal $\log$-likelihood to an expectation over the posterior.
    Letting $\partial_{\beta}$ denote the partial derivative with respect to $\beta$, we have
    \begin{align*}
    \partial_{\beta}\log p(d|\beta) &= \mathbb{E}_{\hat{u}|d,\beta}\left[\partial_{\beta}\log p(\hat{u}|\beta)\right] \\
    &= \mathbb{E}_{\hat{u}|d,\beta}\left[\partial_{\beta}\left(\frac{M}{2}\log\beta -\frac{1}{2}\log\det \Sigma^{-1}_{\mathcal{F}_M} - \frac{\beta}{2}\langle \hat{u}-\hat{u}^0, \Sigma^{-1}_{\mathcal{F}_M}(\hat{u}-\hat{u}^0)\rangle\right)\right] \\
    &= \mathbb{E}_{\hat{u}|d,\beta}\left[ \frac{M}{2\beta} -\frac{1}{2}\langle \hat{u}-\hat{u}^0, \Sigma^{-1}_{\mathcal{F}_M}(\hat{u}-\hat{u}^0)\rangle\right] \\
    &= \frac{M}{2\beta} -\frac{1}{2}\mathbb{E}_{\hat{u}|d,\beta}\left[\langle \hat{u}-\hat{u}^0, \Sigma^{-1}_{\mathcal{F}_M}(\hat{u}-\hat{u}^0)\rangle\right],
    \end{align*}
    which gives the first term.
    What remains is to compute the posterior expectation.
    Application of a standard Gaussian identity yields
    $$
    \mathbb{E}_{\hat{u}|d,\beta}\left[\langle \hat{u}-\hat{u}^0, \Sigma^{-1}_{\mathcal{F}_M}(\hat{u}-\hat{u}^0)\rangle\right] = \langle\Tilde{m} - \hat{u}^0, \Sigma^{-1}_{\mathcal{F}_M}(\Tilde{m} - \hat{u}^0)\rangle + \tr{\Sigma^{-1}_{\mathcal{F}_M}\Tilde{\Sigma}_{\mathcal{F}_M}},
    $$
    which shows the remaining terms.
\end{proof}

\begin{remark}
    The expression we have derived in Lemma~\ref{lem:grad} for $\partial_{\beta}\mathcal{L}(\beta)$ can be seen elsewhere, but written in terms of a quadratic expression of the data $d$.
    For example, this appears in~\cite[Equation 5.9]{williams2006gaussian} or in the function space setting, which is our case, in~\cite{dunlop2020hyperparameter} (although they do not derive the gradient).
    What we have done is state the result in a way in which the dependency of the $\log$-evidence on the posterior mean is made explicit.
    This enables us to asymptotically identify the unique estimate of $\beta$ through application of Theorem~\ref{thm:brownmeanconv} and Theorem~\ref{thm:browncovconv}.
    Whereas the previous works do so only in numerical experiments.
\end{remark}

We now identify the MAP estimate of $\beta$ in the large data limit under different prior choices.
\begin{theorem}
    \label{thm:modelerr}
    Let $\mathcal{L}(\beta) \coloneqq \log\int\frac{1}{Z}p(d | \hat{u})p(\hat{u}|\beta) d\hat{u} + \log p(\beta)$ as given by the hierarchical model eq.~(\ref{eqn:hiearchy}).
    Further, let the assumptions of Theorem~\ref{thm:convergence} hold.
    Then, in the limit $h_{X_n}\to 0$, we have the following.
    \begin{enumerate}
        \item If $\beta$ is assigned a flat prior, then
        $$
        \beta^* = \frac{M}{\langle \hat{u}^* - \hat{u}^0, \Sigma_{\mathcal{F}_M}^{-1}(\hat{u}^*-\hat{u}^0)\rangle}.
        $$
        \item If $\beta$ is assigned Jeffreys prior, then
        $$
        \beta^* = \frac{M -2 }{\langle \hat{u}^* - \hat{u}^0, \Sigma_{\mathcal{F}_M}^{-1}(\hat{u}^*-\hat{u}^0)\rangle}.
        $$
    \end{enumerate}
    Here, $\hat{u}^*$ is the ground truth field which generated the data and $\hat{u}^0$ is the prior mean function, both projected onto $\mathcal{F}$.
\end{theorem}
\begin{proof}
    We have by Lemma~\ref{lem:grad}
    \begin{equation*}
        \frac{\partial}{\partial\beta} \mathcal{L}(\beta) = \frac{M}{2\beta} + \frac{\partial}{\partial\beta}\log p(\beta) - \frac{1}{2}\langle\Tilde{m} - \hat{u}^0,\Sigma^{-1}_{\mathcal{F}_M}(\Tilde{m} - \hat{u}^0)\rangle - \frac{1}{2}\tr{\Sigma^{-1}_{\mathcal{F}_M}\Tilde{\Sigma}_{\mathcal{F}_M}}.
    \end{equation*}
    Now, by Theorem~\ref{thm:brownmeanconv} we have $\Tilde{m} \to u^*$, and by Theorem~\ref{thm:browncovconv} $\Tilde{k} \to 0$ as $h_{X_n}\to 0$.
    Passing to the limit, the gradient becomes
    $$
    \frac{\partial}{\partial\beta} \mathcal{L}(\beta) = \frac{M}{2\beta} + \frac{\partial}{\partial\beta}\log p(\beta) - \frac{1}{2}\langle \hat{u}^* - \hat{u}^0,\Sigma^{-1}_{\mathcal{F}_M}(\hat{u}^* - \hat{u}^0)\rangle.
    $$
    Under a flat prior, $\partial/\partial \beta \log p(\beta) = 0$.
    Setting the gradient to zero, and solving for $\beta$ gives (i).
    Under Jeffreys prior, $p(\beta)\propto 1/\beta$, so $\partial/\partial \beta \log p(\beta) = -1/\beta$, and again setting the gradient to zero, and solving for $\beta$ gives (ii).
\end{proof}

Note that statement (i) of Theorem~\ref{thm:modelerr} is simply the MLE.
Theorem~\ref{thm:modelerr} shows that the MAP estimate of $\beta$ is sensitive to model-form error.
Observe that in each estimate, the term in the denominator is
$$
\langle \hat{u}^* - \hat{u}^0, \Sigma_{\mathcal{F}_M}^{-1}(\hat{u}^*-\hat{u}^0)\rangle = \|\hat{u}^*-\hat{u}^0\|_{H_k}^2,
$$
restricted to $\mathcal{F}$, which was derived in the proof of Proposition~\ref{prop:pinnsloss}.
That is, the optimal value of $\beta$ is inversely proportional to the RKHS distance between the prior mean function $u^0$ and the ground truth $u^*$.
The same holds for $L^2(\Omega)$ distance by the Sobolev embedding theorem.
Recall that $u^0$ is exactly the unique solution to the chosen physical model eq.~(\ref{eqn:poisson}).
Hence, $\beta$ is sensitive to the distance between the true field $u^*$, and the one we have a priori assumed is correct $u^0$.
As $u^*$ moves further away from $u^0$, the optimal value of $\beta$ decreases.
This manifests in larger variance of the samples from the physics-informed prior, as evidenced in Fig.~\ref{fig:exampleprior}, and can be interpreted as a lower level of trust in the assumed physics.
On the other hand, if we have selected the perfect model, i.e. $\|u^*-u^0\|_{H_k}^2 = 0$ then $\beta \to \infty$.
The prior then collapses to a Dirac centered at $u^0$, which signals the absence of model-form error.

Finally, we study how model-form error affects the inverse problem of identifying $q$.
We modify the model eq.~(\ref{eqn:hiearchy}) to
\begin{align*}
    &\beta \sim p(\beta), \quad \hat{q}\sim p(\hat{q}), \nonumber \\
    &\hat{u} | \beta, \hat{q} \sim \mathcal{N}(\hat{u}^0(\cdot;\hat{q}),\beta^{-1}\Sigma_{\mathcal{F}_M}), \nonumber\\
    &d | \hat{u} \sim \mathcal{N}(\hat{u},\sigma^2I),
\end{align*}
where $\hat{q}$ is any parametrization of $q$, e.g. a deep neural network.
We have also explicitly stated the dependency of $\hat{u}^0$ on $\hat{q}$.
As before, the posterior for the inverse problem can be derived with Bayes's rule and taking the marginal:
\begin{equation}
    \label{eqn:inverseprob}
    p(\beta,\hat{q}|d) = \int \frac{1}{Z}p(d|\hat{u})p(\hat{u}|\beta,\hat{q})p(\beta)p(\hat{q})d\hat{u}.
\end{equation}
Since all probabilities involved are Gaussian and the measurement is linear, eq.~(\ref{eqn:inverseprob}) has a known analytical form
\begin{equation}
    \label{eqn:inversesoln}
    p(\beta,\hat{q}|d) \propto \mathcal{N}(d| \hat{u}^0(\cdot;\hat{q}),\beta^{-1}\Sigma_{\mathcal{F}_M} + \sigma^2I)p(\beta)p(\hat{q}).
\end{equation}
Observe how the variance of eq.~(\ref{eqn:inversesoln}) changes according to the MAP estimate of $\beta$ in Theorem~\ref{thm:modelerr}.
A model with relatively high error will result in a smaller value of $\beta$.
This can result from either choosing the wrong PDE, or by incorrectly identifying the source.
In this situation, the variance in the prediction over $\hat{q}$ increases.

The intuition here is that if the model is wrong, the posterior obtained from the methodology responds with a lower confidence in the prediction of $\hat{q}$.
Likewise, if the model-form error is low, the posterior reflects a higher degree of confidence in the prediction.
This behavior is typically absent from Bayesian methods, as the posterior variance is invariant to model-form error.
Also of note is that the posterior variance never entirely disappears due to the presence of measurement noise.
This agrees with the usual result that identifying the source term of the Poisson equation is an ill-posed inverse problem~\cite{ghattas2021learning}.

\section{Conclusions and outlook}

In this work, we established a connection between the variational formulation of the Poisson equation and GP regression.
Specifically, we showed the resulting inverse problem can be viewed as a kernel method.
Then, from the connections between kernel methods and GP regression, we showed that the loss function provides the MAP estimator for GP regression when starting with a Brownian bridge prior.
In one-dimension, we may even move beyond GP regression and consider the prior as a Gaussian measure on $L^2(\Omega)$, and on the dual of a Sobolev space in the general case.
This is in an effort to incorporate nonlinear measurement modalities into the framework.

Using the connection to GP regression, we studied different properties of the field reconstruction problem.
In Section~\ref{sec:analysis}, we were able to prove convergence of the GP MAP estimator to the ground truth in the limit of infinite data.
This also provides the result for the related variational problem.
We briefly discussed the consequences of this in the context of PINNs.
We also derived a finite-dimensional basis representation of the prior as a subset of $L^2(\Omega)$.
This is in contrast to the usual approach taken in GP regression, which instead learns the posterior on a mesh of $\Omega$.
We proved that this representation and the corresponding posterior approximation converge to the correct Gaussian measure.

The main results of the paper are in Section~\ref{sec:model}, where we connect the method to the important problem of identifying model-form error.
When we work under a physics-informed framework, we a priori assume the system is modeled by a specific form of the physics, which in this case is eq.~(\ref{eqn:poisson}).
In any given application, it is entirely possible that we have picked the wrong model.
The usual paradigm enforces the physics as a hard constraint and does not take this into account.
We have modified the method so that the physics is enforced as a soft constraint.
This is done through inclusion of the hyperparameter $\beta$.

In Theorem~\ref{thm:modelerr}, we showed that when $\beta$ is learned via a MAP estimate, it is sensitive to this model-form error.
As the model-form error increases, the optimal value of $\beta$ adjusts accordingly.
This has the affect of increasing the variance in the samples from the prior, which corresponds to a smaller a priori trust in the physical model we have selected.
We also showed this impacts the variance in the posterior over the source term if we are solving the inverse problem.

While the main focus of this work was on the Poisson equation, it is possible to extend the results to certain other PDEs.
The main requirement is that the PDE can be cast as a variational problem which admits a quadratic positive-definite form.
This is so that it may be connected to a kernel method, from which we define a suitable GP.
Another example one could study is the Helmholtz equation
\begin{equation}
    \label{eqn:helmholtz}
    -\nabla^2 u + \omega^2u + q = 0.
\end{equation}
One can show that eq.~(\ref{eqn:helmholtz}) with Dirichlet boundary conditions has the energy functional~\cite{brezis2011functional}
$$
E(u) = \int_{\Omega} \frac{1}{2}\|\nabla u\|^2 + \frac{k^2}{2}\|u\|^2 + qu\:d\Omega,
$$
which by completing the square becomes
$$
E(u) = \frac{1}{2}\langle u - Cq, C^{-1}(u-Cq)\rangle + \mathrm{const.},
$$
where, $C$ is operator defined by by the Green's function of eq.~(\ref{eqn:helmholtz}).
In 1D, this is the square RKHS-norm with kernel
$$
k(x,x') = 2\sum_{n\in\N} \frac{\sin(n\pi x)\sin(n\pi x')}{n^2\pi^2 - \omega^2}.
$$
Therefore, it appears this trick is limited to linear PDEs so that the Greens function may be identified.
Also, in in much of the literature, it is common to use the integrated square residual of the PDE to define the loss function, rather than a variational form.
In a future work, we plan to extend this method both to nonlinear PDEs and to loss functions defined by the integrated square residual.
This can be done via Taylor approximation.

Lastly, we restricted our work to theory, and did not touch on any numerical methods.
While standard GP regression techniques may be used in applications, there are some computational issues which should be resolved.
The main bottleneck is the fact that the mean function of the physics-informed prior is given by the solution to the PDE.
If we are solving the inverse problem, then the mean function will change every time $q$ is updated, meaning that the PDE must be resolved.
Note that this is also the case in classical Bayesian inverse problems~\cite{stuart2010inverse}, where the physics is enforced in the likelihood.
We plan to address this issue in future work by developing specialized sampling algorithms which avoid needing to call a PDE solver.
This is based on the finite-dimensional basis representation derived in this work.

\bibliographystyle{plain}
\bibliography{references} 

\appendix

\section{Gaussian measures}
\label{apdx:GM}
We summarize important concepts related to Gaussian measures on separable Hilbert spaces.
Note that the theory of Gaussian measures on Hilbert spaces can easily be extended to the Banach space setting, but this is not needed in this work.
The texts~\cite{kuo2006gaussian,bogachev1998gaussian} along with the notes provided in~\cite{eldredge2016analysis} provide a nice background to the theory.

Similarly to GPs, Gaussian measures on Hilbert spaces are defined using \emph{covariance operators}.
For a linear operator $C:H\to H$ to be a valid covariance operator of any Borel measure on a Hilbert space $H$, it must be self-adjoint and positive semi-definite.
However, there is an important restriction when working in infinite-dimensions, namely that for a Gaussian measure on a Hilbert space, the covariance operator must be trace class.
\begin{definition}[Trace class operator]
    A linear operator $C:H\to H$ is said to be trace class if, for any orthonormal basis $(\psi_n)_{n=1}^{\infty}$ of $H$, we have
    $$
    \tr{C} \coloneqq \sum_{n\in\N} \langle \psi_n,C\psi_n\rangle < \infty,
    $$
    where the sum is independent of the choice of basis.
\end{definition}
\begin{remark}
    When $C$ is self-adjoint, we can choose the basis in the above definition to be the eigenfunctions of $C$ in which case $\tr{C} = \sum_{n=1}^{\infty} \lambda_n$, where $\lambda_n$, $n=1,2,\dots,$ are the corresponding eigenvalues.
\end{remark}

Now, let $H$ be a real, separable Hilbert space, and let $\mathcal{B}(H)$ denote the Borel $\sigma$-algebra generated by the open subsets of $H$.
Given a Borel measure $\mu$ on $H$, we first define the notion of its mean function and covariance operator.
\begin{definition}[Mean function and covariance operator]
    \label{def:meancov}
    Let $\mu$ be a Borel measure on $H$.
    The mean function of $\mu$ is the element $m \in H$ such that
    $$
    \langle u, m\rangle = \int_{H} \langle u, z\rangle\mu(dz), \quad \forall u\in H.
    $$
    The covariance operator of $\mu$, denoted by $C$, is the operator which satisfies
    $$
    \langle u, Cv\rangle = \int_{H} \langle u,z\rangle \langle v,z \rangle \mu(dz), \quad \forall u,v\in H.
    $$
\end{definition}
Let $\mu$ and $\nu$ be two Borel measures on $H$.
Then, $\mu$ is said to be \emph{absolutely continuous} with respect to $\nu$ if $\nu(A) = 0$ implies $\mu(A) = 0$ for all $A \in \mathcal{B}(H)$.
We denote this by $\mu \ll \nu$.
Two such measures are said to be \emph{equivalent} if $\mu \ll \nu$ and $\nu \ll \mu$.
Measures which are supported on disjoint sets are called \emph{singular}.

A Borel measure $\mu$ on $H$ is said to be \emph{Gaussian} if, for each $u\in H$, the measurable function $\langle u, \cdot \rangle$ is normally distributed.
That is, there exist $m_u,\sigma_u \in \mathbb{R}$, $\sigma_u \geq 0$, such that
$$
    \mu\left(\{v \in H : \langle u, v \rangle\leq a\}\right) = \int_{-\infty}^a \frac{1}{\sqrt{2\pi \sigma_u}}\exp\left\{-\frac{1}{2\sigma_u} (x-m_u)^2\right\}dx.
$$   
We allow for the case $\sigma_u = 0$, which is a Dirac mass centered at $m_u$.
A Gaussian measure on $H$ is guaranteed to have a well-defined mean and covariance operator given by Definition~\ref{def:meancov}, therefore we are justified in denoting the measure as $\mu \sim \mathcal{N}(m,C)$.
Note that for a Gaussian measure defined on a Banach space, it is necessarily the case that $\tr{C} < \infty$.
The inverse of $C$ is called the precision operator, which we denote by $L$.

Gaussian measures are often characterized by their characteristic functions.
For a Borel measure $\mu$ on $H$, we define the \emph{characteristic function} $\phi$ of $\mu$ by
$$
\phi(u) = \int_{H} \exp\{i\langle u, z\rangle\}\mu(dz), \quad u \in H.
$$
If $\phi$ and $\psi$ are respectively the characteristic functions of the Borel measures $\mu$ and $\nu$ on $H$, and $\phi(u) = \psi(u)$ for all $u\in H$, then $\mu=\nu$.
We have the following two theorems related to characteristic functions of Gaussian measures:
\begin{theorem}[Theorem 6.4~\cite{stuart2010inverse}]
    Let $\mu \sim \mathcal{N}(m,C)$ be a Gaussian measure on $H$.
    Then the characteristic function of $\mu$ is given by $\phi(u) = \exp\left\{ i\langle m, u\rangle -\frac{1}{2}\langle u, Cu \rangle \right\}.$
\end{theorem}
\begin{theorem}[Theorem 2.3~\cite{kuo2006gaussian}]
    \label{thm:charfunc}
    Let $m \in H$ and $C$ be a trace class, positive definite, and self-adjoint operator on $H$.
    Then $\phi(u) = \exp\left\{ i\langle m, u\rangle -\frac{1}{2}\langle u, Cu \rangle \right\}$ is the characteristic function of a Gaussian measure on $H$.
\end{theorem}
The above results show that a Gaussian measure on $H$ is uniquely determined by its mean function and covariance operator.
Further, it is no sacrifice to characterize the measure by its characteristic function.
An important space when working with a Gaussian measure is the associated \emph{Cameron-Martin space}, typically denoted by $E$.
If $\mu \sim \mathcal{N}(0,C)$ is defined on a Hilbert space $H$, then E is defined to be the intersection of all linear spaces with full $\mu$-measure.
On a Hilbert space, $E = \mathrm{range}(C^{1/2})$.

Just as with the RKHS of a GP, the Cameron-Martin space of a Gaussian measure characterizes important behavior of the measure.
In fact, the reproducing kernel of a RKHS is often viewed as the kernel of the covariance operator of a Gaussian measure on $L^2(\Omega)$.
In the setting of Gaussian measures, the two are the same.
The immediate consequence is that sample paths will a.s. not lie in the Cameron-Martin space.
For example, if $\mu$ is the classical Wiener measure on the unit interval, then $E = \{u\in H^1([0,1]):u(0)=0\}$, and $\mu(E) = 0$.
This is precisely the statement that sample paths from the Wiener measure are a.s. not differentiable, which is a well-known result.
Further, the Cameron-Martin space also provides necessary and sufficient conditions for equivalence of Gaussian measures:
\begin{theorem}[Theorem 1~\cite{minh2024infinite}]
    \label{thm:equiv}
    Let $\mu\sim \mathcal{N}(m_1,C_1)$ and $\nu\sim \mathcal{N}(m_2,C_2)$ be two Gaussian measures defined on a Hilbert space.
    Then $\mu$ and $\nu$ are either equivalent or singular.
    They are equivalent if and only if the following two conditions are satisfied:
    \begin{enumerate}
        \item $m_2 - m_1 \in \mathrm{range}(C_{1}^{1/2})$.
        \item There exists a symmetric, Hilbert-Schmidt operator $S$ on $H$, without the eigenvalue $1$, with $C_2 = C_1^{1/2}(I - S)C_1^{1/2}$.
    \end{enumerate}  
\end{theorem}
In Theorem~\ref{thm:equiv}, if $\mu$ and $\nu$ share the same covariance operator $C$, then condition (ii) is immediately satisfied by taking $S = 0$.
One only needs to verify whether or not the shift in mean lives in the Cameron-Martin space.
Therefore it often becomes easier to assess properties of a Gaussian measure by centering it, provided the shift lives in the Cameron-Martin space.

We will also make use of the Wasserstein distance as a metric between probability measures when studying convergence.
Unlike other probabilistic metrics, the Wasserstein distance does not require absolute continuity between distributions, making it a bit more general.
The definition is given in terms of the metric on $H$, which we will take to be the metric induced by the norm $\|\cdot\|_H$.
Let $\mathcal{P}_p(H)$ denote the space of all Borel probability measures on $H$ with finite $p$-th moment.
Then, the $p$-Wassertein distance is defined as
$$
W_p(\mu,\nu) \coloneqq \inf_{\gamma\in\Gamma(\mu,\nu)}\left(\int_{H\times H} \|u-v\|_H^p \gamma(dudv)\right)^{1/p},
$$
for two probability measures $\mu,\nu\in \mathcal{P}_p(H)$ and $\Gamma(\mu,\nu)$ is the set of all couplings of $\mu$ and $\nu$.
In the case that the measures are Gaussian, there exists a useful identity, known as Gelbrich's formula~\cite[Theorem 3.5]{gelbrich1990formula}.
This states that for $\mu = \mathcal{N}(m_1,C_1)$ and $\nu = \mathcal{N}(m_2,C_2)$ (both on $H$), the $2$-Wasserstein distance is given by
\begin{equation}
    W_2(\mu,\nu) = \sqrt{\|m_1-m_2\|^2_H + \tr{C_1} + \tr{C_2} - 2\tr{\sqrt{C_1^{1/2}C_2C_1^{1/2}}}}.
\end{equation}

\section{Proof of Lemma~\ref{lem:bridgerkhs}}
\label{apdx:B}

Fix $x' \in \Omega$, with $\Omega = [0,1]^d$.
Let $x = (x_1,\dots,x_d)\in\Omega$.
It is obvious if any $x_1,\dots,x_d$ is $0$ or $1$, then $k(x,x') = 0$.
To show $k \in H^1(\Omega)$, define the partial sum
\begin{align*}
    k^S(x,x') = 2^d\sum_{|\alpha|=1}^S\frac{\sin(\alpha\pi x)\sin(\alpha\pi x')}{\pi^2|\alpha|^2},
\end{align*}
which is uniformly Lipschitz continuous for any order $S$.
By Mercer's theorem, $\lim_{S \to \infty} k^S = k$ absolutely and uniformly.
To show $k$ is also Lipschitz, we must bound the Lipschitz constant uniformly for any $S$.

Because the convergence is absolute and uniform, $\exists M > 0$ such that for any $x \in \Omega$,
\begin{align*}
    2^d \sum_{|\alpha|\in\N} \frac{|\sin(\alpha \pi x)\sin(\alpha \pi x')|}{\pi^2|\alpha|^2} \leq M.
\end{align*}
Now for $x,y$ in $\Omega$,
\begin{align*}
    |k^S(x,x') - k^S(y,x')| &= \left|2^d\sum_{|\alpha|=1}^S\frac{\left[\sin(\alpha \pi x) - \sin(\alpha \pi y)\right]\sin(\alpha \pi x')}{\pi^2|\alpha|^2}\right| \\
    &\leq 2^d\sum_{|\alpha|=1}^S\left|\frac{\left[\sin(\alpha \pi x) - \sin(\alpha \pi y)\right]\sin(\alpha \pi x')}{\pi^2|\alpha|^2}\right| \\
    &= 2^d\sum_{|\alpha|=1}^S\frac{\left|\sin(\alpha \pi x) - \sin(\alpha \pi y)\right| |\sin(\alpha \pi x')|}{\pi^2|\alpha|^2} \\
    &\leq 2^{d+1}\sum_{|\alpha|=1}^S \frac{|\sin(\alpha \pi x')|}{\pi^2|\alpha|} \|x-y\| \leq 2M\|x-y\|.
\end{align*}
As the Lipschitz constant of $k^S$ is bounded for any $S$, and $\lim_{S\to\infty}k^S = k$ uniformly, $k$ is also Lipschitz continuous.
Hence, $k$ is weakly differentiable.
We have shown that $k\in H_k$, satisfying property (i) of Definition~\ref{def:rkhs}.
    
Next, we prove that $k$ is the reproducing kernel for $H_k$.
Pick $u \in H_k$.
To show that $k$ has the reproducing property on $H_k$, we must have $\langle u, k(\cdot,x')\rangle_{H_k} = u(x')$.
The Mercer representation allows us to write the $H_k$-inner product in terms of $L^2(\Omega)$-inner products, i.e.
$$
\langle u, k(\cdot, x')\rangle_{H_k} = \sum_{|\alpha| \in \N}\lambda_{\alpha}^{-1}\langle u, \psi_{\alpha} \rangle \langle k(\cdot,x'), \psi_{\alpha} \rangle,
$$
where $\psi_{\alpha}$ is any orthonormal basis.
Pick the basis to be the d-dimensional Fourier sine series $\psi_{\alpha} = 2^{d/2}\sin(\alpha \pi x)$.
We can expand $u$ by $u = \sum_{\alpha\in\N^d}\langle u, \psi_{\alpha}\rangle \psi_{\alpha}$.
Now, we have for any fixed $\alpha$
\begin{align*}
    \left\langle k(\cdot,x'), 2^{d/2}\sin(\alpha\pi\cdot)\right\rangle & = \int \left(\sum_{|\gamma|\in\N}2^d\lambda_{\gamma}\sin(\gamma\pi x)\sin(\gamma\pi x')\right)\left(2^{d/2}\sin(\alpha\pi x) \right)dx \\
    &= \int \sum_{|\gamma|\in\N}\left\{2^{3d/2}\lambda_{\gamma}\sin(\gamma\pi x)\sin(\alpha\pi x)\sin(\gamma \pi x')\right\}dx\\
    &= \sum_{|\gamma| \in \N}\left\{2^{d/2}\lambda_{\gamma}\sin(\gamma\pi x')\int 2^d \sin(\gamma \pi x) \sin(\alpha \pi x) dx\right\} \\
    &= 2^{d/2} \lambda_{\alpha}\sin(\alpha\pi x'),
\end{align*}
where the last line holds as $\langle 2^d \sin(\gamma\pi x),\sin(\alpha\pi x)\rangle = 1$ for $\alpha = \gamma$ and $0$ otherwise (it is the orthonormal basis we picked).

Returning to the $H_k$-inner product and inserting the above expression yields
\begin{align*}
    \langle u, k(\cdot,x')\rangle_{H_k} &= \sum_{|\alpha|\in\N} \lambda_{\alpha}^{-1}\langle u, \psi_{\alpha}\rangle\langle k(\cdot,x'),\psi_{\alpha}\rangle \\
    &= \sum_{|\alpha|\in\N} \lambda_{\alpha}^{-1}\langle u, \psi_{\alpha}\rangle 2^{d/2}\lambda_{\alpha}\sin(\alpha\pi x') \\
    &= \sum_{|\alpha|\in\N}\langle u,\psi_{\alpha}\rangle2^{d/2}\sin(\alpha\pi x')\\
    &= \sum_{|\alpha|\in\N}\langle u,\psi_{\alpha}\rangle\psi_{\alpha}(x') = u(x').
\end{align*}
This shows requirement (ii) of Definition~\ref{def:rkhs} also holds, and $k$ is the unique reproducing kernel for $H_k$.

\section{Approximation in the Schwartz space}
\label{apdx:C}

While Theorem~\ref{thm:generalpostconv} is valid only in $H^{-\tau}$ (so that the Gaussian measure is well-defined), we can provide a similar convergence condition in the setting of L\'evy white noises (generalized random fields) for $d>1$, provided that we let $\Omega = \mathbb{R}^d$.
The drawback of working with the space $H^{-\tau}$ is that one must choose a specific value of $\tau$.
Thus, the stochastic process described by the prior does not inherently live on a single separable Hilbert space, and a choice must be made on the underlying ambient space.
In contrast, when working with L\'evy white noises, there is no need to worry about ensuring that the covariance form is trace-class.
This choice is perhaps the most natural or intrinsic way to view the prior, as there is no need to pick a specific $\tau$, as the characterization as a L\'evy white noise always exists. 
The drawback, however, is that in this viewpoint, the prior is not a Gaussian probability measure, hence there is no natural Wasserstein metric and Bayes's theorem no longer holds.
Note that GP regression is still possible in this situation.

Recall the Schwartz space of smooth, rapidly decaying functions
$$
\mathcal{S}(\mathbb{R}^d) = \left\{ u \in C^{\infty}(\mathbb{R}^d) : \forall m \in \mathbb{N}, \alpha\in \mathbb{N}^d, \sup_{x\in \mathbb{R}^d}(1+|x|)^m \left|D^{\alpha}u(x)\right| < \infty\right\}.
$$
Let $\mathcal{S}'(\mathbb{R}^d)$ denote the dual of $\mathcal{S}(\mathbb{R}^d)$.
Note that $\mathcal{S}'(\mathbb{R}^d)$ is known as the space of tempered distributions, due to the test function topology of $\mathcal{S}\left(\mathbb{R}^d\right)$.
We can rely on L\'evy's continuity theorem~\cite{bierme2017generalized} in the setting of tempered distributions.
This yields a parallel convergence theorem to Theorem~\ref{thm:generalpostconv} in $\mathcal{S}'(\mathbb{R}^d)$.

\begin{theorem}
    \label{thm:convergence2}
    Let $d>1$, $(\psi_i)_{i\in\N}$ be an orthonormal basis for $L^2(\mathbb{R}^d)$, and for each $n\in\N$, let $\mu_{n}$ be given by eq.~(\ref{eqn:finiteprior}).
    Then there exists a L\'evy white noise $\mu$ on $\mathcal{S}'(\mathbb{R}^d)$ such that $\mu_{n}$ converges in distribution to $\mu$ under the strong topology.
    Further, $\mu$ has characteristic function $\phi_{\mu}(u) = \exp\left\{-\frac{1}{2}\langle u, Cu\rangle\right\}$, $u\in \mathcal{S}(\mathbb{R}^d)$, hence $\mu$ can be regarded as a Gaussian random field over $\mathcal{S}'(\mathbb{R}^d)$.
\end{theorem}

\begin{remark}
    The space $\mathcal{S}'(\mathbb{R}^d)$ is inconveniently large in a machine learning context.
    For instance, the Dirac delta distribution belongs to $\mathcal{S}'(\mathbb{R}^d)$, and we cannot make sense of generating such a sample.
    However, there are often smaller function spaces with full measure under a L\'evy white noise which are easier to characterize.
    A general methodology for identifying a Besov space where this property holds is provided in~\cite{aziznejad2020wavelet}.
    In this case, the distribution is simply the $d$-dimensional Brownian sheet, conditioned to be zero on the boundaries.
    One can show this process is continuous on the unit cube~\cite{adler2010geometry}.
\end{remark}

The proof is quite involved and requires a fair bit of background.
First, recall the following definitions related to random fields.
\begin{definition}[Random field, L\'evy white noise]
    Let $(\Omega, \mathcal{F}, \mu)$ be a probability space, and $U\subseteq\mathbb{R}^d$ an open set.
    \begin{enumerate}
        \item A random field $X$ on $U$ is a measurable mapping $X:U\times \Omega \to \mathbb{R}^n$ such that for any $x\in U$, $X(x;\cdot)$ is a real-valued random vector.
        \item A L\'evy white noise $X$ (generalized random field) is a measurable mapping $X:\mathcal{S}'(\mathbb{R}^d)\times \Omega\to \mathbb{R}^n$. That is, $X(u)$ is a real-valued random vector $\forall u \in \mathcal{S}(\mathbb{R}^d)$.
    \end{enumerate}
\end{definition}
The characteristic function of a L\'evy white noise is defined as the Fourier transform of $X$, i.e.,
$$
\phi_{X}(u) = \mathbb{E}[\exp\{iX(u)\}] = \int_{\mathcal{S}'(\mathbb{R}^d)}\exp\{iL(u)\}d\mu_X(L),\quad u\in \mathcal{S}(\mathbb{R}^d).
$$

The main idea of the proof is to apply the following version of L\'evy's continuity theorem in the setting of tempered distributions.
\begin{theorem}[L\'evy's continuity theorem~\cite{bierme2017generalized}]
    \label{thm:levy}
    Let $(X_n)_{n\in\mathbb{N}}$ be a collection of L\'evy white noises, each with characteristic function $\phi_{X_n}$.
    Suppose that there exists a function $\phi: \mathcal{S}(\mathbb{R}^d)\to\mathbb{C}$, which is continuous at $0$, such that $\phi_{X_n} \to \phi$ pointwise.
    Then there exists a L\'evy white noise $X$, with characteristic function $\phi$, such that $X_n \overset{d}{\to} X$ under the strong topology.
\end{theorem}

To apply Theorem~\ref{thm:levy}, we must show that the finite-dimensional measure given by eq.~(\ref{eqn:finiteprior}) admits a L\'evy white noise representation.
For this, we can rely on the Minlos-Bochner theorem.
\begin{theorem}[Minlos-Bochner theorem~\cite{bierme2017generalized}]
    \label{thm:minlos}
    Let $\phi : \mathcal{S}(\mathbb{R}^d) \to \mathbb{C}$ with $\phi(0) = 1$ be positive-definite and continuous at $0$.
    Then there exists a L\'evy white noise $X$ defined on some probability space $(\Omega, \mathcal{F}, \mu)$ such that $\phi$ is the characteristic function of $X$.
\end{theorem}

\begin{remark}
    We clarify the precise meaning of the convergence appearing in Theorem~\ref{thm:levy}.
    We will use $L$ to denote elements of $\mathcal{S}'(\mathbb{R}^d)$.
    The strong topology $\tau$ on $\mathcal{S}'(\mathbb{R}^d)$ is generated by the collection of semi-norms
    $$
    q_B(L) = \sup_{u\in B} |L(u)|, \quad B\subset \mathcal{S}(\mathbb{R}^d),
    $$
    where $B$ is bounded.
    Now, let $(X_n)_{n\in\N}$ and $X$ be L\'evy white noises with measures $(\mu_n)_{n\in\N}$ and $\mu$.
    We say that $X_n$ converges to $X$ in distribution under the strong topology, denoted by $X_n \overset{d}{\to} X$, if 
    $$
    \lim_{n \to \infty} \int_{\mathcal{S}'(\mathbb{R}^d)} F(L)d\mu_{X_n}(L) = \int_{\mathcal{S}'(\mathbb{R}^d)} F(L)d\mu(L), \quad \forall F \in \mathcal{C}_b(\mathcal{S}'(\mathbb{R}^d),\tau),
    $$
    with $\mathcal{C}_b(\mathcal{S}'(\mathbb{R}^d),\tau)$ being the space of bounded continuous forms on $\mathcal{S}'(\mathbb{R}^d)$ under $\tau$.
\end{remark}

We can now prove the following.
\begin{lemma}
    \label{lem:apdx1}
    For each $\mu_n$ as given by eq.~(\ref{eqn:finiteprior}), the associated random field $X_n$ admits a version which is a L\'evy white noise.
\end{lemma}
\begin{proof}
    Restrict each $\mu_n$ to $\mathcal{S}(\mathbb{R}^d)$, which can be done as $\mathcal{S}(\mathbb{R}^d) \subset L^2(\mathbb{R}^d)$.
    Then, each $\mu_n$ is a Gaussian measure associated with a Gaussian random field $X_n$ on $\mathcal{S}(\mathbb{R}^d)$.
    Since $X_n$ is Gaussian, it has characteristic function
    $$
    \phi_{X_n}(u) = \exp\left\{-\frac{1}{2}\langle \hat{u}, \Sigma_{\mathcal{F}_n}\hat{u}\rangle\right\}, \quad u\in \mathcal{S}(\mathbb{R}^d),
    $$
    where $\hat{u} = \sum_{j=1}^n\langle u,\psi_j\rangle\psi_j$ given an orthonormal basis $(\psi_j)_{j=1}^n \subset \mathcal{F}_n$.
    Observe that $\phi_{X_n}(0) = 1$ and note that $\phi_{X_n}$ is positive-definite as any characteristic function is positive-definite.
    We also have $\langle \cdot, \Sigma_{\mathcal{F}_n}\cdot\rangle = \|\cdot\|_{\Sigma_{\mathcal{F}_n}}^2$, and recalling that any norm is continuous, we identify that $\phi_{X_n}$ is a composition of continuous functions.
    Hence $\phi_{X_n}$ is continuous at $0$, and application of the Minlos-Bochner theorem completes the proof.
\end{proof}

Finally, we will look to apply Theorem~\ref{thm:levy}.
By Lemma~\ref{lem:apdx1}, we may regard each $\mu_n$ in the sequence as a L\'evy white noise.
As they are Gaussian, each has characteristic function
$$
    \phi_{X_n}(u) = \exp\left\{-\frac{1}{2}\langle \hat{u}, \Sigma_{\mathcal{F}_n}\hat{u}\rangle\right\}, \quad u\in \mathcal{S}(\mathbb{R}^d).
$$
Let $\phi_{\mu}(u) = \exp\left\{-\frac{1}{2}\langle u,Cu\rangle\right\}$, which is continuous at $0$.
Then in the limit, for any $u\in \mathcal{S}(\mathbb{R}^d)$, we have
$$
\lim_{n\to\infty} \phi_{X_n}(u) = \phi(u).
$$
By L\'evy's continuity theorem, there exists a L\'evy white noise $\mu$ on $\mathcal{S}'(\mathbb{R}^d)$ with characteristic function $\phi_{\mu}$ such that $\mu_n\overset{\mathrm{d}}{\rightarrow}\mu$.
Finally, $\phi_{\mu}$ is the form of a characteristic function of a Gaussian random field on $\mathcal{S}'(\mathbb{R}^d)$, which completes the proof of Theorem~\ref{thm:convergence2}.

\end{document}